\documentclass[twoside,11pt]{article}
\usepackage{mathrsfs}
\usepackage{amsfonts}
\usepackage{bbm}

\usepackage{jmlr2e}
\usepackage{mathrsfs,multirow,hhline}
\def\RR{\mathbb{R}}

\def\prob{\ensuremath{\, \mathrm{Prob} \,}}

\jmlrheading{}{2015}{1-??}{/; Revised /}{/}{Shao-Bo Lin, Xin Guo,
and Ding-Xuan Zhou}

\ShortHeadings{Distributed Learning}{Lin, Guo and Zhou} \firstpageno{1}

\begin{document}

\title{Distributed Learning with Regularized Least Squares}

\author{\name Shao-Bo Lin \email sblin1983@gmail.com \\
        \addr Department of Mathematics \\
         City University of Hong Kong \\
         Tat Chee Avenue, Kowloon, Hong Kong, China \\
\name Xin Guo \email x.guo@polyu.edu.hk \\
        \addr Department of Applied Mathematics \\
         The Hong Kong Polytechnic University \\
         Hung Hom, Kowloon, Hong Kong, China \\
        \name Ding-Xuan Zhou \email mazhou@cityu.edu.hk \\
        \addr Department of Mathematics \\
         City University of Hong Kong \\
         Tat Chee Avenue, Kowloon, Hong Kong, China}

\editor{}


\maketitle

\begin{abstract}%
We study distributed learning with the least squares regularization scheme in a reproducing kernel Hilbert space (RKHS). By a divide-and-conquer approach, the algorithm partitions a data set into disjoint data subsets, applies the least squares regularization scheme to each data
subset to produce an output function, and then takes an average of the individual output functions as a final global estimator or predictor. We show with error bounds in expectation in both the $L^2$-metric and RKHS-metric that the global output function of this distributed learning
is a good approximation to the algorithm processing the whole data in one single machine. Our error bounds are sharp and stated in a general setting without any eigenfunction assumption. The analysis is achieved by a novel second order decomposition of operator differences in our integral operator approach. Even for the classical least squares regularization scheme in the RKHS associated with a general kernel, we give the best learning rate in the literature.
\end{abstract}


\section{Introduction and Distributed Learning Algorithms}

In the era of big data, the rapid expansion of computing capacities in automatic data
generation and acquisition brings data of unprecedented size and
complexity, and raises a series of scientific challenges such as
storage bottleneck and algorithmic scalability \citep{Zhou2014}.
To overcome the difficulty, some approaches for generating scalable approximate algorithms have been introduced
in the literature such as low-rank
approximations of kernel matrices for kernel principal component analysis
\citep{Scholkopf1998, Bach2013}, incomplete Cholesky decomposition
\citep{Fine2002}, early-stopping of iterative optimization algorithms
for gradient descent methods \citep{Yao2007, Raskutti2014}, and
greedy-type algorithms. Another method proposed recently is distributed learning
based on a divide-and-conquer approach and a particular learning algorithm implemented in individual machines \citep{Zhang2014, Shamir2014}.
This method produces distributed learning algorithms consisting of three steps: partitioning the data into
disjoint subsets, applying a particular learning algorithm implemented in an individual machine to each data subset
to produce an individual output (function), and synthesizing a global output by utilizing some average of the individual outputs. This method can successfully reduce the time
and memory costs, and its learning performance has been observed in many practical applications to be as good as that of a big machine which could process the whole data.
Theoretical attempts have been recently made in \citep{Zhang2013, Zhang2014} to derive learning rates for distributed learning with least squares regularization under certain assumptions.

This paper aims at error analysis of the distributed learning with
regularized least squares and its approximation to the algorithm
processing the whole data in one single machine. Recall
\citep{Taylor2004, Evgeniou2000} that in a reproducing kernel
Hilbert space (RKHS) $({\mathcal H}_K, \|\cdot\|_K)$ induced by a
Mercer kernel $K$ on an input metric space ${\mathcal X}$, with a
sample $D=\{(x_i, y_i)\}_{i=1}^N \subset {\mathcal X} \times
{\mathcal Y}$ where ${\mathcal Y}=\RR$ is the output space, the
least squares regularization scheme can be stated as
\begin{equation}\label{All estimate}
    f_{D,\lambda} =\arg\min_{f\in \mathcal{H}_{K}}
    \left\{\frac{1}{|D|}\sum_{(x, y)\in D}(f(x)-y)^2+\lambda\|f\|^2_{K}\right\}.
\end{equation}
Here $\lambda >0$ is a regularization parameter and $|D|=:N$ is the cardinality of $D$. This
learning algorithm is also called kernel ridge regression in statistics and has been well studied in learning theory. See e.g. \citep{Devito2005, Caponnetto2007, SteinwartHS, BPR, Smale2007, Steinwart2008}.
The regularization scheme (\ref{All estimate}) can be explicitly
solved by using a standard matrix inversion technique, which
requires costs of $\mathcal O(N^3)$ in time and $\mathcal O(N^2)$ in
memory. However, this matrix inversion
technique may not conquer challenges on storages or computations arising from big data.

The distributed learning algorithm studied in this paper starts with partitioning the data set $D$ into $m$ disjoint subsets $\{D_j\}_{j=1}^m$. Then it assigns each data subset $D_j$ to
one machine or processor to produce a local
estimator $f_{D_j, \lambda}$ by the least squares regularization scheme (\ref{All estimate}). Finally, these local estimators are communicated to a central
processor, and a global estimator $\overline{f}_{D, \lambda}$ is
synthesized by taking a weighted average
\begin{equation}\label{distrilearn}
\overline{f}_{D, \lambda} = \sum_{j=1}^m \frac{|D_j|}{|D|} f_{D_j, \lambda}
\end{equation}
of the local estimators $\{f_{D_j, \lambda}\}_{j=1}^m$. This algorithm has been studied with a matrix analysis approach in \citep{Zhang2014} where some error analysis has been conducted under some eigenfunction assumptions for the integral operator associated with the kernel, presenting error bounds in expectation.

In this paper we shall use a novel integral operator approach to prove that $\overline{f}_{D, \lambda}$ is a good approximation of $f_{D, \lambda}$.
We present a representation of the difference $\overline{f}_{D, \lambda} - f_{D, \lambda}$ in terms of empirical integral operators,
and analyze the error $\overline{f}_{D, \lambda} - f_{D, \lambda}$ in expectation without any eigenfunction assumptions. As a by-product, we present the best learning rates for the least squares regularization scheme (\ref{All estimate}) in a general setting, which surprisingly has not been done for a general kernel in the literature (see detailed comparisons below).

\section{Main Results}

Our analysis is carried out in the standard least squares regression framework with a Borel probability measure $\rho$ on ${\mathcal Z} :={\mathcal X}\times  {\mathcal Y}$, where the input space $\mathcal{X}$ is a compact metric space. The sample $D$ is independently drawn according to $\rho$. The Mercer kernel $K: {\mathcal X}\times {\mathcal X} \to \RR$ defines an integral operator $L_K$ on ${\mathcal H}_K$ as
\begin{equation}\label{integraloper}
L_K(f) =\int_{\mathcal X} K_x f(x)d\rho_X, \qquad f\in {\mathcal H}_K,
\end{equation}
where $K_x$ is the function $K(\cdot, x)$ in ${\mathcal H}_K$ and $\rho_X$ is the marginal distribution of $\rho$ on ${\mathcal X}$.

\subsection{Error Bounds for the Distributed Learning Algorithm}

Our error bounds in expectation for the distributed learning algorithm (\ref{distrilearn}) require the uniform boundedness condition for the output $y$, that is, for some constant $M >0$, there holds $|y| \leq M$ almost surely.
Our bounds are stated in terms of the approximation error
\begin{equation}\label{approxerror}
\|f_\lambda-f_\rho\|_\rho,
\end{equation}
where $f_{\lambda}$ is the data-free limit of (\ref{All estimate}) defined by
\begin{equation}\label{flambda}
    f_{\lambda} =\arg\min_{f\in \mathcal{H}_{K}}
    \left\{\int_{\mathcal Z} (f(x) - y)^2 d \rho +\lambda \|f\|^2_{K}\right\},
\end{equation}
$\|\cdot\|_\rho$ denotes the norm of $L^2_{\rho_{_X}}$, the Hilbert space of square
integrable functions with respect to $\rho_X$, and $f_\rho$ is the regression function (conditional mean) of $\rho$ defined by
$$ f_\rho (x) = \int_{{\mathcal Y}} y d\rho(y|x), \qquad x\in {\mathcal X}, $$
with $\rho(\cdot|x)$ being the conditional distribution of $\rho$ induced at $x\in {\mathcal X}$.

Since $K$ is continuous, symmetric and positive semidefinite,  $L_K$
is a compact positive operator of trace class and $L_{K} +\lambda I$
is  invertible. Define a quantity measuring the complexity of
${\mathcal H}_K$ with respect to $\rho_X$, the {\sl effective
dimension} \citep{Zhang2005}, to be the trace of the operator $(L_K + \lambda
I)^{-1} L_K$ as
\begin{equation}\label{effectdim}
 {\mathcal N}(\lambda) = \hbox{Tr}\left((L_K + \lambda I)^{-1} L_K\right), \qquad \lambda >0.
\end{equation}
In Section \ref{proofs}  we shall prove the following first main
result of this paper concerning error bounds in expectation of
$\overline{f}_{D, \lambda} - f_{D, \lambda}$ in ${\mathcal H}_K$ and
in $L^2_{\rho_{_X}}$. Denote $\kappa=\sup_{x\in\mathcal X}\sqrt{K(x,
x)}$.

\begin{theorem}\label{Mainexpected}
Assume $|y| \leq M$ almost surely. If $|D_j| = \frac{N}{m}$ for $j=1, \ldots, m$, then we have
\begin{eqnarray*}
E\left\|\overline{f}_{D, \lambda} - f_{D, \lambda}\right\|_\rho \leq C_\kappa \left(\frac{m}{(N \lambda)^2} + \frac{{\mathcal N}(\lambda)}{N\lambda}\right) \sqrt{m} \biggl\{\frac{\|f_\lambda-f_\rho\|_\rho}{\sqrt{N\lambda}} m^{2} +  M \sqrt{\lambda} \left\{1+ \frac{m^2}{(N \lambda)^2} + \frac{m{\mathcal N}(\lambda)}{N\lambda}\right\}\biggr\}
\end{eqnarray*}
and
\begin{eqnarray*}
E\left\|\overline{f}_{D, \lambda} - f_{D, \lambda}\right\|_K \leq C_\kappa \left(\frac{m}{(N \lambda)^2} + \frac{{\mathcal N}(\lambda)}{N\lambda}\right) \sqrt{m} \biggl\{\frac{\|f_\lambda-f_\rho\|_\rho}{\sqrt{N}} \frac{m^{2}}{\lambda} +  M \left\{1+ \frac{m^2}{(N \lambda)^2} + \frac{m{\mathcal N}(\lambda)}{N\lambda}\right\}\biggr\},
\end{eqnarray*}
where $C_{\kappa}$ is a constant depending only on $\kappa$.
\end{theorem}

To derive the explicit learning rate of algorithm
(\ref{distrilearn}), one needs the following assumption
as a characteristic of the complexity of the hypothesis space
\citep{Caponnetto2007,Blanchard2010},
\begin{equation}\label{Assumption on effecdim}
     \mathcal N(\lambda)\leq c\lambda^{-\beta}, \qquad  \forall
     \lambda>0
\end{equation}
for some $0<\beta\leq 1$ and   $c>0$.
In particular, let $\{(\lambda_l,
\phi_l)\}_{l}$ be a set of normalized eigenpairs of $L_K$ on
$\mathcal H_K$ with $\{\phi_l\}_{l=1}^\infty$ being an
orthonormal basis of $\mathcal H_K$
and $\{\lambda_l\}_{l=1}^\infty$ arranged in a non-increasing order, and let
$$
L_K=\sum_{\ell=1}^\infty \lambda_\ell\langle\cdot,\phi_{\ell}\rangle_K \phi_{\ell}
$$
be the spectral decomposition. Since $\mathcal{N}(\lambda)=
\sum_l\frac{\lambda_l}{\lambda_l+\lambda}\leq
\sum_l\frac{\lambda_l}{\lambda}=\mathrm{Tr}(L_K)/\lambda$,
the condition (\ref{Assumption on effecdim}) with $\beta=1$
always holds true with $c=\mbox{Tr}(L_K)\leq\kappa^2$. For $0<\beta<1$,
$\lambda_n\leq c'n^{-1/\beta}$ implies (\ref{Assumption on
effecdim}) (see, e.g. \cite{Caponnetto2007}).
This condition $\lambda_n\leq c'n^{-1/\beta}$ is satisfied, e.g.,
by the Sobolev space $W^{m_*}(B(\mathbb{R}^d))$, where $B(\mathbb{R}^d)$
is a ball in $\mathbb{R}^d$ with the integer $m_*>d/2$, $\rho_X$ being
the uniform distribution on $B(\mathbb{R}^d)$, and $\beta=\frac{d}{2m_*}$
\citep{SteinwartHS,EdmundsTriebel}.

 The
results in \citep{Caponnetto2007,SteinwartHS,Zhang2014} showed that
if $\frac{{\mathcal N}(\lambda)}{N\lambda} =O(1)$, then the optimal
learning rates of algorithm (\ref{distrilearn}) with $m=1$ can be
obtained in the sense that the upper and lower bounds for
$\max_{f_\rho\in\mathcal H_K}E\left\|f_{D, \lambda}-
f_\rho\right\|_\rho$ are asymptomatically identical. Thus, to derive
 learning rates for $E\left\|\overline{f}_{D, \lambda} - f_{D,
\lambda}\right\|_\rho$, a more general case with an arbitrary $m$ is
covered as follows.

\begin{corollary}\label{MainEqualsize}
Assume $|y| \leq M$ almost surely. If $|D_j| = \frac{N}{m}$ for $j=1, \ldots, m$, and $\lambda$ satisfies
\begin{equation}\label{lambdacond}
0< \lambda \leq C_0 \quad \hbox{and} \quad \frac{m {\mathcal N}(\lambda)}{N\lambda} \leq C_0,
\end{equation}
for some constant $C_0>0$, then we have
\begin{eqnarray*}
E\left\|\overline{f}_{D, \lambda} - f_{D, \lambda}\right\|_\rho \leq \widetilde{C}_\kappa \frac{m{\mathcal N}(\lambda)}{N\lambda} \biggl(\|f_\lambda-f_\rho\|_\rho  \frac{m \sqrt{m}}{\sqrt{N\lambda}} +  M \frac{\sqrt{\lambda}}{\sqrt{m}}\biggr)
\end{eqnarray*}
and
$$
E\left\|\overline{f}_{D, \lambda} - f_{D, \lambda}\right\|_K
\leq \widetilde{C}_\kappa \frac{m{\mathcal N}(\lambda)}{N\lambda} \biggl(\|f_\lambda-f_\rho\|_\rho \frac{m \sqrt{m}}{\sqrt{N}\lambda} +  \frac{M}{\sqrt{m}}\biggr),$$
where $\widetilde{C}_\kappa$ is a constant depending only on $\kappa$, $C_0$, and the largest eigenvalue of $L_K$.
\end{corollary}

In the special case  that $f_\rho \in {\mathcal H}_K$, the
approximation error can be bounded as $\|f_\lambda-f_\rho\|_\rho
\leq \|f_\rho\|_K \sqrt{\lambda}$. A more general condition can be
imposed for the regression function as
\begin{equation}\label{approxr}
f_\rho = L_K^{r}(g_\rho) \quad \hbox{for some} \ g_\rho \in L^2_{\rho_X}, \ r>0,
\end{equation}
where the integral operator $L_{K}$  is regarded as a compact
positive operator on $L^2_{\rho_X}$ and its $r$th power is well
defined for any $r>0$. The condition (\ref{approxr}) means $f_\rho$
lies in the range of $L_K^{r}$, and the special case $f_\rho \in
{\mathcal H}_K$ corresponds to the choice $r=1/2$. Under condition
(\ref{approxr}), we can obtain from  Corollary \ref{MainEqualsize}
the following nice convergence rates for the distributed learning
algorithm.

\begin{corollary}\label{SpecialEqualsize}
Assume (\ref{approxr}) for some $0<r \leq 1$, $|y| \leq M$ almost surely and ${\mathcal N}(\lambda) = O(\lambda^{-\frac{1}{2\alpha}})$ for some $\alpha >0$. If $|D_j| = \frac{N}{m}$ for $j=1, \ldots, m$ with
\begin{equation}\label{mrestrict}
m \leq N^{\frac{1 + 2 \alpha \max\{2r-1, 0\} + 2 \alpha (2r -1)}{4 + 8 \alpha \max\{2r, 1\}- 4 \alpha + 4 \alpha r}},
\end{equation}
and $\lambda = \left(\frac{m}{N}\right)^{\frac{2\alpha}{2 \alpha \max\{2r, 1\} +1}}$, then we have
\begin{eqnarray*}
E\left\|\overline{f}_{D, \lambda} - f_{D, \lambda}\right\|_\rho =O\left(N^{-\frac{\alpha + 2 \alpha \max\{2r-1, 0\}}{2 \alpha \max\{2r, 1\} +1}} m^{-\frac{\frac{1}{2} - \alpha \max\{2r-1, 0\}}{2 \alpha \max\{2r, 1\} +1}}\right)
\end{eqnarray*}
and
$$
E\left\|\overline{f}_{D, \lambda} - f_{D, \lambda}\right\|_K
=O\left(N^{-\frac{2 \alpha \max\{2r-1, 0\}}{2 \alpha \max\{2r, 1\}
+1}} m^{-\frac{\frac{1}{2} + 2 \alpha - \alpha \max\{2r, 1\}}{2
\alpha \max\{2r, 1\} +1}}\right). $$ In particular, when $f_\rho \in
{\mathcal H}_K$  and $m \leq N^{\frac{1}{4 + 6 \alpha}}$, the choice
$\lambda = \left(\frac{m}{N}\right)^{\frac{2\alpha}{2 \alpha +1}}$
 yields $ E\left\|\overline{f}_{D, \lambda} - f_{D,
\lambda}\right\|_\rho =O\left(N^{-\frac{\alpha}{2 \alpha  +1}}
m^{-\frac{1}{4 \alpha +2}}\right) $ and $ E\left\|\overline{f}_{D,
\lambda} - f_{D, \lambda}\right\|_K
=O\left(\frac{1}{\sqrt{m}}\right).$
\end{corollary}

\begin{remark}\label{Remark:1}
In Corollary \ref{SpecialEqualsize}, we present learning rates in
both $\mathcal H_K$   and $L_{\rho_X}^2$ norms. The
$L_{\rho_X}^2$-norm bound is useful because it equals (subject to a constant) the
generalization error $\int_{\mathcal Z}(f(x)-y)^2d\rho$.
The $\mathcal H_K$ norm
controls the $L_{\rho_X}^2$ norm since for any $f$ in $\mathcal{H}_K$,
$\|f\|_\rho\leq \|f\|_\infty\leq \kappa\|f\|_K$ \citep{Smale2007};
this inequality
also implies the application of the $\mathcal{H}_K$-norm
bound in the mismatched problem where the generalization
power is measured in some $L^2_\mu$-norm with $\mu$ different from
$\rho_X$.
\end{remark}

\begin{remark}\label{Remark:2}
In Corollary \ref{SpecialEqualsize}, the established error bounds
are monotonously decreasing with respect to $m$, which is different
from the error analysis in \citep{Zhang2014}. The  reason is that we
are concerned with  the difference between $\overline{f}_{D,
\lambda}$ and $f_{D, \lambda}$. This difference reflects the
variance of the distributed learning algorithm. Concerning the
learning rate (as shown in Corollary \ref{Finalrate} below), the
regularization parameter $\lambda$ should be smaller, and then the
learning rate is independent of $m$, provided $m$ is not very large.
\end{remark}

\subsection{Minimax Rates of Convergence for Least Squares Regularization Scheme}

The second main result of this paper is a sharp error bound for the least squares regularization scheme (\ref{All estimate}). We can even relax the uniform boundedness to a moment condition that for some constant $p\geq 1$,
\begin{equation}\label{uniformbound}
\sigma^2_\rho \in L^{p}_{\rho_X},
\end{equation}
where $\sigma^2_\rho$ is the conditional variance defined by $\sigma^2_\rho(x) = \int_{{\mathcal Y}} \left(y - f_\rho (x)\right)^2 d\rho(y|x)$.

The following learning rates for the least squares regularization scheme (\ref{All estimate}) will be proved in Section \ref{proofsConf}. The existence of $f_\lambda$ is ensured by $E[y^2]<\infty$.

\begin{theorem}\label{MainConfid}
Assume $E[y^2]<\infty$ and (\ref{uniformbound}) for some $1 \leq p \leq \infty$. Then we have
\begin{eqnarray}
&&E\left[\left\|f_{D, \lambda} - f_{\rho}\right\|_\rho\right] \leq \left(1 +
 59\kappa^4  +59\kappa^2\right) (1 +\kappa) \left(1 +\frac{1}{(N\lambda)^2} +\frac{{\mathcal N}(\lambda)}{N\lambda}\right)  \nonumber \\
&& \quad \biggl\{\biggl(1 +   \frac{1}{\sqrt{N\lambda}}\biggr) \|f_\lambda-f_\rho\|_\rho + \sqrt{\left\|\sigma^2_\rho\right\|_p}  \left(\frac{{\mathcal N}(\lambda)}{N}\right)^{\frac{1}{2} (1-\frac{1}{p})} \left(\frac{1}{N\lambda}\right)^{\frac{1}{2p}}\biggr\}. \label{confboundGen}
\end{eqnarray}
\end{theorem}

If the parameters satisfy $\frac{{\mathcal N}(\lambda)}{N\lambda} =O(1)$, we have the following explicit bound.

\begin{corollary}\label{Specialpara}
Assume $E[y^2]<\infty$ and (\ref{uniformbound}) for some $1 \leq p \leq \infty$. If $\lambda$ satisfies (\ref{lambdacond}) with $m=1$, then we have
$$
E\left[\left\|f_{D, \lambda} - f_{\rho}\right\|_\rho\right] = O\left(\|f_\lambda-f_\rho\|_\rho + \left(\frac{{\mathcal N}(\lambda)}{N}\right)^{\frac{1}{2} (1-\frac{1}{p})} \left(\frac{1}{N\lambda}\right)^{\frac{1}{2p}}\right).
$$
In particular, if $p=\infty$, that is, the conditional variances are uniformly bounded, then
$$
E\left[\left\|f_{D, \lambda} - f_{\rho}\right\|_\rho\right] = O\left(\|f_\lambda-f_\rho\|_\rho + \sqrt{\frac{{\mathcal N}(\lambda)}{N}}\right).
$$
\end{corollary}

In particular, when (\ref{approxr}) is satisfied, we have the following learning rates.

\begin{corollary}\label{MainMinmax}
Assume $E[y^2]<\infty$, (\ref{uniformbound}) for some $1 \leq p \leq \infty$, and (\ref{approxr}) for some $0< r \leq 1$. If ${\mathcal N}(\lambda) = O(\lambda^{-\frac{1}{2\alpha}})$ for some $\alpha >0$,
then by taking $\lambda = N^{-\frac{2\alpha}{2 \alpha \max\{2r, 1\} +1}}$ we have
$$
E\left[\left\|f_{D, \lambda} - f_{\rho}\right\|_\rho\right] = O\left(N^{- \frac{2r \alpha}{2 \alpha \max\{2r, 1\} +1} + \frac{1}{2p}  \frac{2 \alpha -1}{2 \alpha \max\{2r, 1\} +1}}\right).
$$
In particular, when $p=\infty$ (the conditional variances are uniformly bounded), we have
$$
E\left[\left\|f_{D, \lambda} - f_{\rho}\right\|_\rho\right] = O\left(N^{-\frac{2 r \alpha}{2 \alpha \max\{2r, 1\} +1}}\right).
$$
\end{corollary}

\begin{remark}\label{Remark:3}
For $r\in[\frac12,1]$,  \citep{Caponnetto2007,SteinwartHS}
give the minimax lower bound $N^{-\frac{4 r \alpha}{4 \alpha r  +1}}$
for $E[\|f_{D, \lambda} - f_{\rho}\|_\rho^2]$ as $p\to\infty$.
So the convergence rate we obtain in Corollary \ref{MainMinmax}
is sharp.
\end{remark}

Combining bounds for $\left\|\overline{f}_{D, \lambda} - f_{D,
\lambda}\right\|_\rho$ and $\left\|f_{D, \lambda} -
f_{\rho}\right\|_\rho$, we can derive learning rates for the
distributed learning algorithm (\ref{distrilearn}) for regression.

\begin{corollary}\label{Finalrate}
Assume $|y| \leq M$ almost surely and (\ref{approxr})  for some
$\frac{1}{2}< r \leq 1$. If ${\mathcal N}(\lambda) =
O(\lambda^{-\frac{1}{2\alpha}})$ for some $\alpha >0$, $|D_j| =
\frac{N}{m}$ for $j=1, \ldots, m$, and $m$ satisfies the restriction
\begin{equation}\label{restrictmfinal}
m \leq N^{\min \left\{\frac{6 \alpha(2r-1) +1}{5(4 \alpha r +1)}, \frac{2\alpha(2r -1)}{4 \alpha r +1}\right\}},
\end{equation}
then by taking
$\lambda = N^{-\frac{2\alpha}{4\alpha r +1}}$, we have
$$ E\left[\left\|\overline{f}_{D, \lambda} - f_{\rho}\right\|_\rho\right] = O\left(N^{-\frac{2 \alpha r}{4 \alpha r +1}}\right). $$
\end{corollary}

\begin{remark}\label{Remark:4}
Corollary \ref{Finalrate} shows that distributed learning with least
squares regularization scheme (\ref{distrilearn}) can reach the
minimax rates in expectation, provided $m$ satisfies
(\ref{restrictmfinal}).  It should be pointed out that we consider
error analysis under (\ref{approxr}) with $1/2<r\leq 1$ while
\citep{Zhang2014} focused on the case (\ref{approxr}) with $r=1/2$.
The main novelty of our analysis is that by using a novel second
order decomposition for the difference of operator inverses, we
remove the eigenfunction assumptions in \citep{Zhang2014} and
provide error bounds for a larger range of $r$.
\end{remark}

\begin{remark}\label{Remark:5}
In this paper, we only derive minimax rates for the least squares
regularization scheme (\ref{All estimate}) as well as its
distributed version (\ref{distrilearn}) in expectation. We guess it
is possible to derive error bounds in probability by combining the
proposed second order decomposition  approach   with the analysis in
\citep{Caponnetto2007,Blanchard2010}. We will study it in a future
publication.
\end{remark}

\begin{remark}\label{Remark:6}
Corollary \ref{Finalrate} and Corollary \ref{MainMinmax}
suggest that the optimal choice of the regularization parameter $\lambda$
should be independent of the number $m$ of partitions. In particular,
for regularized least squares (\ref{All estimate}),
the distributed scheme shares the optimal $\lambda$ with the batch learning
scheme. This observation is consistent with the results
in \citep{Zhang2014}. We note that there are several parameter
selection approaches in literature including cross-validation
\citep{Gyorfi2002} and the balancing principle \citep{Devito2010}.
It would be interesting to develop some parameter selection
method for distributed learning.
\end{remark}

\section{Comparisons and Discussion}\label{compare}

The least squares regularization scheme (\ref{All estimate}) is a classical algorithm for regression and has been extensively investigated in statistics and learning theory. There is a vast literature on this topic. Here for a general kernel beyond the Sobolev kernels, we compare our results with the best learning rates in the existing literature. Denote the set of positive eigenvalues of $L_K$ as $\{\lambda_i\}_i$ arranged in a decreasing order, and a set of normalized (in ${\mathcal H}_K$) eigenfunctions $\{\varphi_i\}_{i}$ of $L_K$ corresponding to the eigenvalues $\{\lambda_i\}_i$.

Under the assumption that the orthogonal projection $f_{\mathcal H}$ of $f_\rho$ in $L^2_{\rho_X}$ onto the closure of ${\mathcal H}_K$ satisfies (\ref{approxr}) for some $\frac{1}{2}\leq r \leq 1$ , and that the eigenvalues $\lambda_i$ satisfy  $\lambda_i \approx i^{-2 \alpha}$ with some $\alpha > 1/2$, it was proved in \citep{Caponnetto2007} that
$$ \lim_{\tau \to \infty} \limsup_{N \to \infty} \sup_{\rho \in {\mathcal P}(\alpha)} \prob \left[\left\|f_{D, \lambda_N} - f_{\mathcal H}\right\|^2_{\rho} > \tau \lambda_N^{2r}\right] =0, $$
where
$$ \lambda_N =\left\{\begin{array}{ll} N^{-\frac{2\alpha}{4 \alpha r +1}}, & \hbox{if} \ \frac{1}{2} < r \leq 1, \\
\left(\frac{\log N}{N}\right)^{\frac{2\alpha}{2 \alpha +1}}, &
\hbox{if} \ r= \frac{1}{2}, \end{array}\right. $$ and ${\mathcal
P}(\alpha)$  denotes a set of probability measures $\rho$ satisfying
some moment decay condition (which is satisfied when  $|y| \leq M$).
This learning rate is suboptimal due to the limitation taken for
$\tau \to \infty$ and the logarithmic factor in the case
$r=\frac{1}{2}$. In particular, to have $\left\|f_{D, \lambda_N} -
f_{\mathcal H}\right\|^2_{\rho} \leq \tau_\eta \lambda_N^{2r}$ with
confidence $1-\eta$, one needs to restrict $N \geq N_\eta$ to be
large enough and has the constant $\tau_\eta$ depending on $\eta$
 to
be large enough. Using similar mathematical tools as that in
(\citep{Caponnetto2007}) and a novel second order decomposition for
the
  difference of operator inverses, we succeed in deriving  learning rates in expectation
in Corollary \ref{MainMinmax} by removing the  logarithmic factor in
the case $r=\frac{1}{2}$.

Under the assumption that $|y| \leq M$ almost surely,  the
eigenvalues $\lambda_i$ satisfying  $\lambda_i \leq a i^{-2 \alpha}$
with some $\alpha > 1/2$ and $a>0$, and for some constant $C>0$, the
pair $(K, \rho_X)$ satisfying
\begin{equation}\label{interpolationC}
\|f\|_\infty \leq C \|f\|_K^{\frac{1}{2 \alpha}} \|f\|_{\rho}^{1-\frac{1}{2 \alpha}}
\end{equation} for every $f\in {\mathcal H}_K$, it was proved in \citep{SteinwartHS} that for some constant $c_{\alpha, C}$ depending only on $\alpha$ and $C$, with confidence $1-\eta$, for any $0< \lambda \leq 1$,
$$ \left\|\pi_M\left(f_{D, \lambda}\right) - f_\rho\right\|^2_{\rho} \leq 9 {\mathcal A}_2 (\lambda) + c_{\alpha, C} \frac{a^{1/(2\alpha)} M^2 \log(3/\eta)}{\lambda^{1/(2\alpha)} N}.  $$
Here $\pi_M$ is the projection onto the interval $[-M, M]$ defined \citep{CWYZ, WYZ2006} by
$$ \pi_M (f)(x) = \left\{\begin{array}{ll} M, & \hbox{if} \ f(x) >M, \\
f(x), & \hbox{if} \ |f(x)| \leq M, \\
-M, & \hbox{if} \ f(x) <-M, \end{array}\right. $$
and ${\mathcal A}_2 (\lambda)$ is the approximation error defined by
$$ {\mathcal A}_2 (\lambda) = \inf_{f\in {\mathcal H}_K} \left\{\left\|f-f_\rho\right\|_\rho^2 + \lambda \|f\|_K^2\right\}. $$
When $f_\rho \in {\mathcal H}_K$, ${\mathcal A}_2 (\lambda)  =
O(\lambda)$ and the choice $\lambda_N = N^{\frac{2\alpha}{2 \alpha
+1}}$ gives a learning rate of order $\left\|f_{D, \lambda_N} -
f_\rho\right\|_{\rho} = O\left(N^{-\frac{\alpha}{2 \alpha
+1}}\right)$.
 But one needs to impose the condition
(\ref{interpolationC}) for the functions spaces $L^2_{\rho_X}$ and
${\mathcal H}_K$, and to take the projection onto $[-M, M]$,
although    (\ref{interpolationC}) is more general than the uniform
boundedness assumption of the eigenfunctions and holds when
$\mathcal H_K$ is the Sobolev space and $\rho_X$ is the uniform
distribution \citep{SteinwartHS,Mendelson2010}. Our learning rates
in Corollary \ref{MainMinmax} do not require such a condition for
the function spaces, nor do we take the projection. Learning rates
for the least squares regularization scheme (\ref{All estimate}) in
the ${\mathcal H}_K$-metric have been investigated in the literature
\citep{Smale2007}.

For the distributed learning algorithm (\ref{distrilearn})  with
subsets $\{D_j\}_{j=1}^m$ of equal size, under the assumption that
for some constants $k> 2$ and $A <\infty$, the eigenfunctions
$\{\varphi_i\}_{i}$ satisfy
\begin{equation}\label{eigenfcnC}
\left\|\varphi_i\right\|_{L^{2k}_{\rho_X}}^{2 k} =E\left[\left|\varphi_i (x)\right|^{2k}\right] \leq A^{2 k}, \qquad i=1, 2, \ldots,
\end{equation}
that $f_\rho \in {\mathcal H}_K$ and $\lambda_i \leq a i^{-2
\alpha}$  for some $\alpha > 1/2$ and $a>0$, it was proved in
\citep{Zhang2014} that for $\lambda = N^{\frac{2\alpha}{2 \alpha
+1}}$ and $m$ satisfying the restriction
$$ m \leq c_\alpha \left(\frac{N^{\frac{2(k-4)\alpha -k}{2 \alpha +1}}}{A^{4k} \log^k N}\right)^{\frac{1}{k-2}} $$
with a constant  $c_\alpha$ depending only on $\alpha$, there holds
$E\left[\left\|\overline{f}_{D, \lambda} -
f_\rho\right\|_\rho^2\right] =O\left(N^{-\frac{2\alpha}{2 \alpha
+1}}\right)$. This interesting result was achieved by a matrix
analysis approach for which the eigenfunction assumption
(\ref{eigenfcnC}) played an essential role.

The eigenfunction assumption (\ref{eigenfcnC})  generalizes the
classical case that the eigenfunctions are uniformly bounded:
$\left\|\varphi_i\right\|_{\infty} =O(1)$. An example of a
$C^\infty$ Mercer kernel was presented in \citep{Zhou02, Zhoucap} to
show that smoothness of the Mercer kernel does not guarantee the
uniform boundedness of the eigenfunctions. Furthermore,
\citep{Gittens2016} provided a practical reason for avoiding unform
boundedness  assumption  on the eigenfunctions (or eigenvectors) in
terms of localization and sparseness.
 The condition
(\ref{eigenfcnC}), to the best of our knowledge, only holds when
$\mathcal H_K$ is the Sobolev space and $\rho_X$ is the Lebesgue
measure or $K$ is a periodical kernel. It is a challenge to verify
(\ref{eigenfcnC}) for some widely used kernels including the
Gaussian kernel. It would be interesting to find practical instance
such that (\ref{eigenfcnC}) holds.
 Our learning rates stated
in Corollary \ref{SpecialEqualsize} do not require such an
eigenfunction assumption. Also, our  restriction (\ref{mrestrict})
for the number $m$ of local processors is more general when $\alpha$
is close to $1/2$. Notice that the learning rates stated in
Corollary \ref{SpecialEqualsize} are for the difference
$\overline{f}_{D, \lambda} - f_{D, \lambda}$ between the output
function of the distributed learning algorithm (\ref{distrilearn})
and that of the algorithm (\ref{All estimate}) using the whole data.
In the special case of $r=\frac{1}{2}$, we can see that
$E\left\|\overline{f}_{D, \lambda} - f_{D, \lambda}\right\|_\rho
=O\left(N^{-\frac{\alpha}{2 \alpha +1}} m^{-\frac{1}{4 \alpha
+2}}\right)$, achieved by choosing $\lambda =
\left(\frac{m}{N}\right)^{\frac{2\alpha}{2 \alpha +1}}$, is smaller
as $m$ becomes larger. This is natural because the error
$E\left\|\overline{f}_{D, \lambda} - f_{D, \lambda}\right\|_\rho$
reflects more the sample error and should become smaller when we use
more local processors. On the other hand, as one expects, increasing
the number $m$ of local processors would increase the approximation
error for the regression problem, which can also be seen from the
bound with $\lambda = \left(\frac{m}{N}\right)^{\frac{2\alpha}{2
\alpha +1}}$ stated in Theorem \ref{MainConfid}. The result in
Corollary \ref{Finalrate} with $r> \frac{1}{2}$ compensates and
gives the best learning rate $ E\left[\left\|\overline{f}_{D,
\lambda} - f_{\rho}\right\|_\rho\right] = O\left(N^{-\frac{2 r
\alpha}{4 \alpha r +1}}\right)$ by restricting $m$ as in
(\ref{restrictmfinal}).

Besides the divide-and-conquer technique, there are some
other widely-used approaches towards the goal of reducing
time complexity. For example, the localized learning \citep{Meister2016},
Nystr\"{o}m regularization \citep{Bach2013} and  on-line learning
\citep{Dekel2012}, to name but a few. A key advantage of the
divide-and-conquer technique is that it also reduces the
space complexity without a significant lost (as proved in
this paper) of prediction power. Although here we only consider
the distributed regularized least squares, it would be
important also to develop the theory for the distributed variance of
other algorithms such as the spectral algorithms \citep{BPR},
empirical feature-based learning \citep{GuoZhou1},
error entropy minimization \citep{HFWZ},
randomized Kaczmarz \citep{Linjh2015}, and so on.
It would be important to consider the strategies of parameter
selection and data partition for distributed learning.

In this paper, we consider the regularized least squares with Mercer kernels.
It would be interesting to minimize the assumptions on the kernel and
the domain to maximize the scope of applications. For example,
the domain that does not have a metric \citep{ShenWongXGSPeptide}, the kernel that is
only bounded and measurable \citep{Steinwart2012}, and so on.

\section{Second Order Decomposition of Operator Differences and Norms}\label{errordecom}

To analyze the error $\overline{f}_{D, \lambda} - f_{D, \lambda}$,
we need the following representation in terms of the difference of
inverse operators denoted by
\begin{equation}\label{QdNotation}
 Q_{D(x)}= \left(L_{K, D (x)} +\lambda I\right)^{-1} - \left(L_{K}+\lambda I\right)^{-1}
\end{equation}
and $Q_{D_j(x)}$ for the data subset $D_j$. The empirical integral
operator $L_{K, D_j(x)}$ is defined with $D$ replaced by the data
subset $D_j$.

Define two random variables $\xi_\lambda$ and $\xi_0$ with values in
the Hilbert space ${\mathcal H}_K$ by
\begin{equation}\label{xi0xi}
\xi_0 (z) =\bigl(y - f_\rho (x)\bigr) K_{x}, \quad \xi_\lambda (z) =
\bigl(y - f_\lambda (x)\bigr) K_{x}, \qquad z=(x, y) \in {\mathcal
Z}.
\end{equation}
We can derive a representation for $\overline{f}_{D, \lambda} -
f_{D, \lambda}$ in the following lemma.

\begin{lemma}\label{difference}
Assume $E[y^2] <\infty$. For $\lambda >0$, we have
\begin{eqnarray}
\overline{f}_{D, \lambda} - f_{D, \lambda} &=& \sum_{j=1}^m
\frac{|D_j|}{|D|} \left[\left(L_{K, D_j (x)} +\lambda I\right)^{-1}
- \left(L_{K, D (x)}+\lambda
      I\right)^{-1}\right] \Delta_j \nonumber \\
      &=& \sum_{j=1}^m \frac{|D_j|}{|D|} \left[Q_{D_j(x)}\right] \Delta'_j + \sum_{j=1}^m \frac{|D_j|}{|D|} \left[Q_{D_j(x)}\right] \Delta''_j  - \left[Q_{D(x)}\right] \Delta_D, \label{expressdiffLK}
\end{eqnarray}
where
$$ \Delta_j :=\frac{1}{|D_j|} \sum_{z \in D_j} \xi_\lambda (z) - E[\xi_\lambda], \quad \Delta_D := \frac{1}{|D|} \sum_{z \in D} \xi_\lambda (z) - E[\xi_\lambda], $$
and
$$ \Delta'_j :=\frac{1}{|D_j|} \sum_{z \in D_j} \xi_0 (z), \quad \Delta''_j :=\frac{1}{|D_j|} \sum_{z \in D_j} \left(\xi_\lambda - \xi_0\right) (z) - E[\xi_\lambda]. $$
\end{lemma}

\begin{proof} A well known formula (see e.g. \citep{Smale2007}) asserts that
$$f_{D_j, \lambda} - f_{\lambda}  = \left(L_{K, D_j(x)} +\lambda I\right)^{-1} \Delta_j. $$
So we know that
$$ \overline{f}_{D, \lambda} - f_{\lambda} =
\sum_{j=1}^m \frac{|D_j|}{|D|} \left\{f_{D_j, \lambda} -
f_{\lambda}\right\} = \sum_{j=1}^m \frac{|D_j|}{|D|} \left(L_{K,
D_j(x)} +\lambda I\right)^{-1} \Delta_j. $$ Also, with the whole
data $D$, we have
\begin{equation}\label{fDlambda}
f_{D, \lambda} - f_{\lambda} = \left(L_{K, D(x)} +\lambda
I\right)^{-1} \Delta_D.
\end{equation}
But
$$ \Delta_D = \frac{1}{|D|} \sum_{z \in D} \xi_\lambda (z) - E[\xi_\lambda] = \sum_{j=1}^m \frac{|D_j|}{|D|} \left\{\frac{1}{|D_j|} \sum_{z \in D_j} \xi_\lambda (z) - E[\xi_\lambda]\right\} = \sum_{j=1}^m \frac{|D_j|}{|D|} \Delta_j. $$
Hence
$$ f_{D, \lambda} - f_{\lambda} =
 \sum_{j=1}^m \frac{|D_j|}{|D|} \left(L_{K, D(x)} +\lambda I\right)^{-1} \Delta_j. $$
Then the first desired expression for $\overline{f}_{D, \lambda} -
f_{D, \lambda}$ follows.

By adding and subtracting the operator $\left(L_K+\lambda
I\right)^{-1}$, writing $\Delta_j =\Delta'_j + \Delta''_j$, and
noting $E[\xi_0]=0$, we know that the first expression implies
(\ref{expressdiffLK}). This proves Lemma \ref{difference}.
\end{proof}

Our error estimates are achieved by a novel second order decomposition of operator differences in our integral operator approach. We approximate the integral operator $L_K$ by the empirical integral operator $L_{K, D(x)}$ on ${\mathcal H}_K$
defined with the input data set $D(x) =\{x_i\}_{i=1}^N =\{x: (x, y)\in D \ \hbox{for some} \ y\in {\mathcal Y}\}$ as
\begin{equation}\label{empiricalint}
L_{K, D(x)}(f) = \frac{1}{|D|} \sum_{x \in D(x)} f(x) K_{x} = \frac{1}{|D|} \sum_{x \in D(x)} \langle f, K_{x}\rangle_K  K_{x}, \qquad f\in {\mathcal H}_K,
\end{equation}
where the reproducing property $f(x)= \langle f, K_{x}\rangle_K$ for $f\in {\mathcal H}_K$ is used. Since $K$ is a Mercer kernel,
$L_{K, D_j(x)}$ is a finite-rank positive operator and $L_{K, D_j (x)}+\lambda I$ is invertible.

The operator difference in our study is $A^{-1} - B^{-1}$ with $A= L_{K, D (x)} +\lambda I$ and $B= L_{K}+\lambda I$. Our second order decomposition for the difference $A^{-1} - B^{-1}$ is stated as follows.

\begin{lemma}\label{operatordifference}
Let $A$ and $B$ be invertible operators on a Banach space. Then we have
\begin{equation}\label{operatordiffform}
A^{-1} - B^{-1} = B^{-1}\left\{B - A\right\}B^{-1} + B^{-1}\left\{B - A\right\} A^{-1}\left\{B - A\right\} B^{-1}.
\end{equation}
In particular, we have
\begin{eqnarray}
 \left(L_{K, D (x)} +\lambda I\right)^{-1} - \left(L_{K}+\lambda I\right)^{-1} = \left(L_{K} +\lambda I\right)^{-1} \left\{L_K- L_{K, D(x)}\right\}  \left(L_{K}+\lambda
      I\right)^{-1} \nonumber \\
\quad + \left(L_{K} +\lambda I\right)^{-1} \left\{L_K- L_{K, D(x)}\right\}  \left(L_{K, D(x)}+\lambda
      I\right)^{-1} \left\{L_K- L_{K, D(x)}\right\} \left(L_{K} +\lambda I\right)^{-1}. \label{secondorder}
\end{eqnarray}
\end{lemma}

\begin{proof}
We can decompose the operator $A^{-1} - B^{-1}$ as
\begin{equation}\label{operatordiffII}
 A^{-1} - B^{-1} = B^{-1}\left\{B - A\right\} A^{-1}.
\end{equation}
This is the first order decomposition.

Write the last term $A^{-1}$ as $B^{-1} + (A^{-1} - B^{-1})$ and apply another first order decomposition similar to (\ref{operatordiffII}) as
$$ A^{-1} - B^{-1} = A^{-1}\left\{B - A\right\} B^{-1}. $$
It follows from (\ref{operatordiffII}) that
$$ A^{-1} - B^{-1} = B^{-1}\left\{B - A\right\} \left\{B^{-1} + A^{-1}\left\{B - A\right\} B^{-1}\right\}. $$
Then the desired identity (\ref{operatordiffform}) is verified. The lemma is proved.
\end{proof}

Note that $L_K^{1/2}$ and the $r$th power of the  compact positive
operator $L_{K} +\lambda I$ or $L_{K, D_j (x)}+\lambda I$ is well
defined for any $r\in\RR$. The following lemma which will be proved
in the Appendix provides estimates for the operator $\left(L_{K}
+\lambda I\right)^{-1/2} \left\{L_K- L_{K, D(x)}\right\}$ in the
second order decomposition (\ref{secondorder}). As in
\citep{Caponnetto2007}, we use effective dimensions defined by
(\ref{effectdim}) to estimate operator norms.

\begin{lemma}\label{keynormsoperator}
Let $D$ be a sample drawn independently according to $\rho$. Then the following estimates for the operator norm $\left\|\left(L_{K} +\lambda I\right)^{-1/2} \left\{L_K- L_{K, D(x)}\right\}\right\|$ hold.
\begin{description}
\item{(a)} \quad $E\left[\left\|\left(L_{K} +\lambda I\right)^{-1/2} \left\{L_K- L_{K, D(x)}\right\}\right\|^2\right] \leq \frac{\kappa^2 {\mathcal N}(\lambda)}{|D|}.$

\item{(b)} For any $0< \delta <1$, with confidence at least $1-\delta$, there holds
\begin{equation}\label{normprepareeta2}
\left\|\left(L_{K} +\lambda I\right)^{-1/2} \left\{L_K- L_{K, D(x)}\right\}\right\| \leq {\mathcal B}_{|D|, \lambda} \log \bigl(2/\delta\bigr),
\end{equation}
where we denote the constant
\begin{equation}\label{Bdef}
{\mathcal B}_{|D|, \lambda} =\frac{2\kappa }{\sqrt{|D|}} \left\{\frac{\kappa }{\sqrt{|D|\lambda}} +\sqrt{{\mathcal N}(\lambda)}\right\}.
\end{equation}

\item{(c)} For any $d>1$, there holds
$$
\left\{E\left[\left\|\left(L_{K} +\lambda I\right)^{-1/2}
\left\{L_K- L_{K,
D(x)}\right\}\right\|^d\right]\right\}^{\frac{1}{d}} \leq  (2 d
\Gamma (d)+1)^{\frac{1}{d}} {\mathcal B}_{|D|, \lambda},
$$
where $\Gamma$ is the Gamma function defined for $d>0$ by $\Gamma (d) =\int_0^\infty  u^{d-1} \exp\left\{-u\right\} d u$.
\end{description}
\end{lemma}

To apply (\ref{expressdiffLK}) for  our error analysis, we also need
to bound norms involving $\Delta'_j, \Delta''_j$ and $\Delta_D$. We
are able to give the following estimates even after multiplying with
$\left(L_{K} +\lambda I\right)^{-1/2}$ taken from the operator $Q_{D
(x)}$ or $Q_{D_j (x)}$, which will be proved in the Appendix.

\begin{lemma}\label{keynorms}
Let $D$ be a sample drawn independently according to $\rho$ and $g$ be a measurable bounded function on ${\mathcal Z}$ and $\xi_g$ be the random variable with values on ${\mathcal H}_K$ given by $\xi_g (z) = g(z) K_{x}$ for $z=(x, y) \in {\mathcal Z}$. Then the following statements hold.
\begin{description}
\item{(a)} \quad $E\left[\left\|\left(L_{K} +\lambda I\right)^{-1/2} \left(K_x\right)\right\|_K^2\right] = {\mathcal N}(\lambda).$

\item{(b)} For any $0< \delta <1$, with confidence at least $1-\delta$, there holds
\begin{eqnarray*}
\left\|\left(L_{K} +\lambda I\right)^{-1/2} \left(\frac{1}{|D|} \sum_{z \in D} \xi_g (z) - E\left[\xi_g\right]\right)\right\|_K  \leq \frac{2 \|g\|_\infty \log \bigl(2/\delta\bigr) }{\sqrt{|D|}} \left\{\frac{\kappa }{\sqrt{|D|\lambda}} +\sqrt{{\mathcal N}(\lambda)}\right\}.
\end{eqnarray*}
\end{description}
\end{lemma}

\section{Deriving  of Error Bounds for Least Squares Regularization Scheme}\label{proofsConf}

To illustrate how to apply the second order decomposition (\ref{secondorder}) for operator differences in our integral operator approach, we prove in this section our main result on error bounds for the least squares regularization scheme (\ref{All estimate}).

\begin{proposition}\label{MainPropls}
Assume $E[y^2] <\infty$ and (\ref{uniformbound}) for some $1 \leq p \leq \infty$. Then we have
\begin{eqnarray*}
E\left[\left\|f_{D, \lambda} - f_{\lambda}\right\|_\rho\right] &\leq& \left(1 + 59\kappa^4  +59\kappa^2\right) \left(1 +\frac{1}{(|D|\lambda)^2} +\frac{{\mathcal N}(\lambda)}{|D|\lambda}\right) \\
&& \left\{\kappa^{\frac{1}{p}}\sqrt{\left\|\sigma^2_\rho\right\|_p}  \left(\frac{{\mathcal N}(\lambda)}{|D|}\right)^{\frac{1}{2} (1-\frac{1}{p})} \left(\frac{1}{|D|\lambda}\right)^{\frac{1}{2p}} + \kappa  \frac{\|f_\lambda-f_\rho\|_\rho}{\sqrt{|D|\lambda}}\right\}.
\end{eqnarray*}
\end{proposition}

\begin{proof}
We recall the expression (\ref{fDlambda}) for $f_{D, \lambda} - f_{\lambda}$ and the notation $Q_{D(x)}$ defined by (\ref{QdNotation}) for the operator difference $\left(L_{K, D (x)} +\lambda I\right)^{-1} - \left(L_{K}+\lambda I\right)^{-1}$. Then we see
$$ f_{D, \lambda} - f_{\lambda} = \left[Q_{D(x)}\right] \Delta_D + \left(L_{K}+\lambda
      I\right)^{-1} \Delta_D. $$
To estimate the $L^2_{\rho_X}$ norm, we use the identity
\begin{equation}\label{HKL2}
\|g\|_\rho = \|L_K^{1/2} g\|_K, \qquad \forall g\in L^2_{\rho_X},
\end{equation}
and get
$$
\left\|f_{D, \lambda} - f_{\lambda}\right\|_\rho \leq \left\|L_K^{1/2}\left[Q_{D(x)}\right] \Delta_D\right\|_K + \left\|L_K^{1/2}\left(L_{K}+\lambda
      I\right)^{-1} \Delta_D\right\|_K.
$$

We apply the second order decomposition (\ref{secondorder}), use the bounds $\left\|L_K^{1/2}\left(L_{K} +\lambda I\right)^{-1/2}\right\| \leq 1$,  $\left\|\left(L_{K, D(x)}+\lambda
      I\right)^{-1}\right\| \leq \frac{1}{\lambda}$, $\left\|\left(L_{K} +\lambda I\right)^{-1/2}\right\| \leq 1/\sqrt{\lambda}$, and know that
\begin{eqnarray*}
&& \left\|L_K^{1/2}\left[Q_{D(x)}\right] \Delta_D\right\|_K \leq \left\|\left(L_{K} +\lambda I\right)^{-1/2} \left\{L_K- L_{K, D(x)}\right\} \left(L_{K} +\lambda I\right)^{-1} \Delta_D\right\|_K + \\
&& \left\|\left(L_{K} +\lambda I\right)^{-1/2} \left\{L_K- L_{K, D(x)}\right\}  \left(L_{K, D(x)}+\lambda
      I\right)^{-1} \left\{L_K- L_{K, D(x)}\right\} \left(L_{K} +\lambda I\right)^{-1} \Delta_D\right\|_K \\
&& \leq \left\|\left(L_{K} +\lambda I\right)^{-1/2} \left\{L_K- L_{K, D(x)}\right\}\right\| \frac{1}{\sqrt{\lambda}} \left\|\left(L_{K} +\lambda I\right)^{-1/2} \Delta_D\right\|_K + \\
&& \left\|\left(L_{K} +\lambda I\right)^{-1/2} \left\{L_K- L_{K, D(x)}\right\}\right\| \frac{1}{\lambda} \left\|\left\{L_K- L_{K, D(x)}\right\}\left(L_{K} +\lambda I\right)^{-1/2}\right\| \biggl\|\left(L_{K} +\lambda I\right)^{-1/2} \Delta_D\biggr\|_K.
\end{eqnarray*}
For convenience, we introduce the notation
\begin{equation}\label{XiD}
\Xi_{D} = \left\|\left(L_{K}+\lambda
      I\right)^{-1/2}\left\{L_K- L_{K, D(x)}\right\}\right\|.
\end{equation}
Then the above bound can be restated as
\begin{equation}\label{fDlambdafirst}
\left\|L_K^{1/2}\left[Q_{D(x)}\right] \Delta_D\right\|_K \leq \left(\frac{\Xi_D}{\sqrt{\lambda}} + \frac{\Xi_D^2}{\lambda}\right)\left\|\left(L_{K} +\lambda I\right)^{-1/2} \Delta_D\right\|_K.
\end{equation}
Hence
\begin{eqnarray*}
\left\|f_{D, \lambda} - f_{\lambda}\right\|_\rho \leq \left(1 + \frac{\Xi_D}{\sqrt{\lambda}} + \frac{\Xi_D^2}{\lambda}\right)\left\|\left(L_{K} +\lambda I\right)^{-1/2} \Delta_D\right\|_K,
\end{eqnarray*}
and by the Schwarz inequality we have
\begin{equation}\label{firstdecom}
E\left[\left\|f_{D, \lambda} - f_{\lambda}\right\|_\rho\right] \leq \left\{E\left[\left(1 + \frac{\Xi_D}{\sqrt{\lambda}} + \frac{\Xi_D^2}{\lambda}\right)^2\right]\right\}^{1/2} \left\{E\left[\left\|\left(L_{K} +\lambda I\right)^{-1/2} \Delta_D\right\|_K^2\right]\right\}^{1/2}.
\end{equation}

To deal with the expected value in (\ref{firstdecom}), as in Lemma \ref{difference}, we separate $\Delta_D$ as
$$\Delta_D = \Delta'_D + \Delta_D'', $$
where
$$\Delta'_D :=\frac{1}{|D|} \sum_{z \in D} \xi_0 (z), \quad \Delta''_D :=\frac{1}{|D|} \sum_{z \in D} \left(\xi_\lambda - \xi_0\right) (z) - E[\xi_\lambda]. $$
Then
\begin{eqnarray}
&& \left\{E\left[\left\|\left(L_{K} +\lambda I\right)^{-1/2} \Delta_D\right\|_K^2\right]\right\}^{1/2} \nonumber \\
&&\leq \left\{E\left[\left\|\left(L_{K} +\lambda I\right)^{-1/2} \Delta'_D\right\|_K^2\right]\right\}^{1/2} + \left\{E\left[\left\|\left(L_{K} +\lambda I\right)^{-1/2} \Delta''_D\right\|_K^2\right]\right\}^{1/2}. \label{separate}
\end{eqnarray}
Observe that
$$ \left(L_{K} +\lambda I\right)^{-1/2} \Delta'_D = \sum_{z\in D} \frac{1}{|D|} (y - f_\rho (x)) \left(L_{K} +\lambda I\right)^{-1/2}(K_x). $$
Each term in this expression is unbiased because $\int_{{\mathcal Y}} y- f_\rho (x)  d \rho(y|x) =0$. This unbiasedness and the independence tell us that
\begin{eqnarray}
\left\{E\left[\left\|\left(L_{K} +\lambda I\right)^{-1/2} \Delta'_D\right\|_K^2\right]\right\}^{1/2} &=& \left\{\frac{1}{|D|}
E\left[\left\|(y - f_\rho (x))\left[\left(L_{K} +\lambda I\right)^{-1/2}\right] (K_x)\right\|_K^2\right]\right\}^{1/2} \nonumber \\
&=& \left\{\frac{1}{|D|}
E\left[\sigma^2_\rho (x) \left\|\left[\left(L_{K} +\lambda I\right)^{-1/2}\right] (K_x)\right\|_K^2\right]\right\}^{1/2}. \label{DeltaD'}
\end{eqnarray}

If $\sigma^2_\rho \in L^\infty$, then $\sigma^2_\rho (x) \leq \left\|\sigma^2_\rho\right\|_\infty$ and by Lemma \ref{keynorms} we have
$$ \left\{E\left[\left\|\left(L_{K} +\lambda I\right)^{-1/2} \Delta'_D\right\|_K^2\right]\right\}^{1/2}  \leq \sqrt{\left\|\sigma^2_\rho\right\|_\infty} \sqrt{{\mathcal N}(\lambda)/|D|}. $$

If $\sigma^2_\rho \in L^p_{\rho_X}$ with $1\leq p <\infty$, we take $q= \frac{p}{p-1}$ ($q=\infty$ for $p=1$) satisfying $\frac{1}{p} + \frac{1}{q} =1$ and
apply the H\"older inequality $E[|\xi \eta|] \leq \left(E[|\xi|^p]\right)^{1/p} \left(E[|\eta|^{q}]\right)^{1/q}$ to $\xi = \sigma^2_\rho$, $\eta =\biggl\|\left(L_{K} +\lambda I\right)^{-1/2} (K_x)\biggr\|_K^2$ to find
$$
E\left[\sigma^2_\rho (x) \left\|\left[\left(L_{K} +\lambda I\right)^{-1/2}\right] (K_x)\right\|_K^2\right] \leq \left\|\sigma^2_\rho\right\|_{p}
\left\{E\left[\left\|\left[\left(L_{K} +\lambda I\right)^{-1/2}\right] (K_x)\right\|_K^{2q}\right]\right\}^{1/q}.
$$
But
$$ \left\|\left[\left(L_{K} +\lambda I\right)^{-1/2}\right] (K_x)\right\|_K^{2q-2} \leq \left(\kappa/\sqrt{\lambda}\right)^{2 q -2} $$
and $E\left[\biggl\|\left(L_{K} +\lambda I\right)^{-1/2} (K_x)\biggr\|_K^{2}\right]= {\mathcal N}(\lambda)$ by Lemma \ref{keynorms}. So we have
\begin{eqnarray*}
\left\{E\left[\left\|\left(L_{K} +\lambda I\right)^{-1/2} \Delta'_D\right\|_K^2\right]\right\}^{1/2}  &\leq& \left\{\frac{1}{|D|}
\left\|\sigma^2_\rho\right\|_{p} \left\{\frac{\kappa^{2q -2}}{\lambda^{q-1}} {\mathcal N}(\lambda)\right\}^{1/q}\right\}^{1/2} \\
&=&
\sqrt{\left\|\sigma^2_\rho\right\|_p} \kappa^{\frac{1}{p}} \left(\frac{{\mathcal N}(\lambda)}{|D|}\right)^{\frac{1}{2} (1-\frac{1}{p})} \left(\frac{1}{|D|\lambda}\right)^{\frac{1}{2p}}.
\end{eqnarray*}
Combining the above two cases, we know that for either $p=\infty$ or $p<\infty$,
$$ \left\{E\left[\left\|\left(L_{K} +\lambda I\right)^{-1/2} \Delta'_D\right\|_K^2\right]\right\}^{1/2} \leq
\sqrt{\left\|\sigma^2_\rho\right\|_p} \kappa^{\frac{1}{p}} \left(\frac{{\mathcal N}(\lambda)}{|D|}\right)^{\frac{1}{2} (1-\frac{1}{p})} \left(\frac{1}{|D|\lambda}\right)^{\frac{1}{2p}}.$$

The second term of (\ref{separate}) can be bounded easily as
\begin{eqnarray*}\left\{E\left[\left\|\left(L_{K} +\lambda I\right)^{-1/2} \Delta''_D\right\|_K^2\right]\right\}^{1/2} &\leq& \frac{1}{\sqrt{|D|}}
\left\{E\left[\left(f_\rho (x) - f_\lambda (x)\right)^2
\left\|\left(L_{K} +\lambda I\right)^{-1/2} (K_x) \right\|_K^2\right]\right\}^{1/2} \\
&\leq& \frac{1}{\sqrt{|D|}}
\left\{E\left[\left(f_\rho (x) - f_\lambda (x)\right)^2
\frac{\kappa^2}{\lambda}\right]\right\}^{1/2} = \frac{\kappa \|f_\rho - f_\lambda\|_\rho}{\sqrt{|D|\lambda}}.
\end{eqnarray*}
Putting the above estimates for the two terms of (\ref{separate}) into (\ref{firstdecom}) and applying Lemma \ref{keynormsoperator} to get
\begin{eqnarray*}
\left\{E\left[\left(1 + \frac{\Xi_D}{\sqrt{\lambda}} + \frac{\Xi_D^2}{\lambda}\right)^2\right]\right\}^{1/2} &\leq& 1 + \left\{E\left[\frac{\Xi_D^2}{\lambda}\right]\right\}^{1/2} +
 \left\{E\left[\frac{\Xi_D^4}{\lambda^2}\right]\right\}^{1/2} \\
 &\leq& 1 +\left\{\frac{\kappa^2 {\mathcal N}(\lambda)}{|D|\lambda}\right\}^{1/2} +
 \left\{\frac{49 {\mathcal B}_{|D|, \lambda}^4}{\lambda^2}\right\}^{1/2} \\
 &\leq& 1 +\frac{59\kappa^4}{(|D|\lambda)^2} +\frac{59\kappa^2 {\mathcal N}(\lambda)}{|D|\lambda},
\end{eqnarray*}
we know that $E\left[\left\|f_{D, \lambda} - f_{\lambda}\right\|_\rho\right]$ is bounded by
$$ \left(1 +\frac{59\kappa^4}{(|D|\lambda)^2} +\frac{59\kappa^2 {\mathcal N}(\lambda)}{|D|\lambda}\right) \left(\sqrt{\left\|\sigma^2_\rho\right\|_p} \kappa^{\frac{1}{p}} \left(\frac{{\mathcal N}(\lambda)}{|D|}\right)^{\frac{1}{2} (1-\frac{1}{p})} \left(\frac{1}{|D|\lambda}\right)^{\frac{1}{2p}} + \frac{\kappa \|f_\rho - f_\lambda\|_\rho}{\sqrt{|D|\lambda}}\right). $$
Then our desired error bound follows. The proof of the proposition is complete.
\end{proof}

{\par \noindent{\bf Proof of Theorem \ref{MainConfid}\ }}
Combining Proposition \ref{MainPropls} with the triangle inequality $\|f_{D, \lambda} - f_{\rho}\|_\rho \leq \|f_{D, \lambda} - f_{\lambda}\|_\rho + \|f_\lambda-f_\rho\|_\rho$, we know that
\begin{eqnarray*}
E\left[\left\|f_{D, \lambda} - f_{\rho}\right\|_\rho\right]
&\leq& \|f_\lambda-f_\rho\|_\rho + \left(1 + 59\kappa^4  +59\kappa^2\right) \left(1 +\frac{1}{(|D|\lambda)^2} +\frac{{\mathcal N}(\lambda)}{|D|\lambda}\right) \\
&& \left\{\kappa^{\frac{1}{p}}\sqrt{\left\|\sigma^2_\rho\right\|_p} \left(\frac{{\mathcal N}(\lambda)}{|D|}\right)^{\frac{1}{2} (1-\frac{1}{p})} \left(\frac{1}{|D|\lambda}\right)^{\frac{1}{2p}} + \frac{\kappa}{\sqrt{|D|\lambda}}  \|f_\lambda-f_\rho\|_\rho\right\}.
\end{eqnarray*}
Then the desired error bound holds true, and the proof of Theorem \ref{MainConfid} is complete.
\hfill \BlackBox

\medskip

{\par \noindent{\bf Proof of Corollary \ref{Specialpara}\ }}
By the definition of effective dimension,
$$ {\mathcal N}(\lambda) = \sum_\ell \frac{\lambda_\ell}{\lambda_\ell + \lambda} \geq \frac{\lambda_1}{\lambda_1 + \lambda}. $$
Combining this with the restriction (\ref{lambdacond}) with $m=1$, we find ${\mathcal N}(\lambda) \geq \frac{\lambda_1}{\lambda_1 + C_0}$ and
$N\lambda \geq \frac{\lambda_1}{(\lambda_1 + C_0) C_0}$. Putting these and the restriction (\ref{lambdacond}) with $m=1$ into the error bound (\ref{confboundGen}), we know that
\begin{eqnarray*}
&&E\left[\left\|f_{D, \lambda} - f_{\rho}\right\|_\rho\right] \leq \left(1 + 59\kappa^4  +59\kappa^2\right) (1 +\kappa) \left(1 +\frac{(\lambda_1 + C_0)^2C_0^2}{\lambda_1^2} +C_0\right) \nonumber \\
&&\biggl\{\biggl(1 +   \sqrt{(\lambda_1 + C_0) C_0/\lambda_1}\biggr) \|f_\lambda-f_\rho\|_\rho + \sqrt{\left\|\sigma^2_\rho\right\|_p} \left(\frac{{\mathcal N}(\lambda)}{N}\right)^{\frac{1}{2} (1-\frac{1}{p})} \left(\frac{1}{N\lambda}\right)^{\frac{1}{2p}}\biggr\}.
\end{eqnarray*}
Then the desired bound follows. The proof of Corollary \ref{Specialpara} is complete. \hfill
\BlackBox

To Prove Corollary \ref{MainMinmax}, we need the following bounds
\citep{Smale2007} for  $\|f_\lambda-f_\rho\|_\rho$ and
$\|f_\lambda-f_\rho\|_K$.

\begin{lemma}\label{Lemma: approximation error} Assume (\ref{approxr}) with
 $0<r\leq
1$.  There holds
\begin{equation}\label{approximation error rho}
          \|f_\lambda - f_\rho\|_\rho  \leq \lambda^{r} \|h_\rho\|_\rho.
\end{equation}
Furthermore, if $1/2\leq r\leq 1$, then we have
\begin{equation}\label{approximation error hk}
          \|f_\lambda - f_\rho\|_K  \leq \lambda^{r-1/2} \|h_\rho\|_\rho.
\end{equation}
\end{lemma}

{\par \noindent{\bf Proof of Corollary \ref{MainMinmax}\ }} It
follows from Lemma \ref{Lemma: approximation error} that   the
condition (\ref{approxr}) with $0< r \leq 1$ implies
\begin{eqnarray*}
\|f_\lambda - f_\rho\|_\rho
 \leq \lambda^{r} \|g_\rho\|_\rho.
\end{eqnarray*}
If
$${\mathcal N}(\lambda) \leq C_0 \lambda^{-\frac{1}{2\alpha}}, \qquad \forall \lambda >0 $$
for some constant $C_0 \geq 1$, then the choice $\lambda = N^{-\frac{2\alpha}{2 \alpha \max\{2r, 1\} +1}}$ yields
$$ \frac{{\mathcal N}(\lambda)}{N\lambda} \leq \frac{C_0 \lambda^{-\frac{1}{2\alpha} -1}}{N} = C_0 N^{\frac{2\alpha +1}{2\alpha \max\{2r, 1\}+1}-1} \leq C_0. $$
So (\ref{lambdacond}) with $m=1$ is satisfied. With this choice we also have
\begin{eqnarray*} \left(\frac{{\mathcal N}(\lambda)}{N}\right)^{\frac{1}{2} (1-\frac{1}{p})} \left(\frac{1}{N\lambda}\right)^{\frac{1}{2p}}
&\leq& C_0^{\frac{1}{2} (1-\frac{1}{p})}  N^{- \frac{2 \alpha \max\{2r, 1\}}{2 \alpha \max\{2r, 1\} +1} \frac{1}{2} (1-\frac{1}{p})}  N^{- \frac{2 \alpha \max\{2r, 1\} + 1 - 2 \alpha}{2 \alpha \max\{2r, 1\} +1} \frac{1}{2p}} \\
&=& C_0^{\frac{1}{2} (1-\frac{1}{p})}  N^{- \frac{\alpha \max\{2r, 1\}}{2 \alpha \max\{2r, 1\} +1} + \frac{1}{2p}  \frac{2 \alpha -1}{2 \alpha \max\{2r, 1\} +1}}.
\end{eqnarray*}
Putting these estimates into Corollary \ref{Specialpara}, we know that
\begin{eqnarray*}
E\left[\left\|f_{D, \lambda} - f_{\rho}\right\|_\rho\right]
&=& O\left(N^{-\frac{2\alpha r}{2 \alpha \max\{2r, 1\}+1}} + N^{- \frac{\alpha \max\{2r, 1\}}{2 \alpha \max\{2r, 1\} +1} + \frac{1}{2p}  \frac{2 \alpha -1}{2 \alpha \max\{2r, 1\} +1}}\right) \\
&=& O\left(N^{- \frac{\alpha \min\left\{2 r, \max\{2r, 1\}\right\}}{2 \alpha \max\{2r, 1\} +1} + \frac{1}{2p}  \frac{2 \alpha -1}{2 \alpha \max\{2r, 1\} +1}}\right).
\end{eqnarray*}
But we find
$$ \min\left\{2r, \max\{2r, 1\}\right\} = 2 r $$
by discussing the two different cases $0< r < \frac{1}{2}$ and $\frac{1}{2} \leq r \leq 1$.
Then our conclusion follows immediately. The proof of Corollary \ref{MainMinmax} is complete. \hfill
\BlackBox

\section{Proof of Error Bounds for the Distributed Learning Algorithm}\label{proofs}

In this section, we prove our first main result on the error $\overline{f}_{D, \lambda} - f_{D, \lambda}$ in the ${\mathcal H}_K$ metric and $L^2_{\rho}$ metric. The following result is more general, allowing different sizes for data subsets $\{D_j\}$.

\begin{theorem}\label{Mainexpectedgeneral}
Assume that for some constant $M>0$, $|y|\leq M$ almost surely. Then we have
\begin{eqnarray*}
&& E\left[\left\|\overline{f}_{D, \lambda} - f_{D, \lambda}\right\|_\rho\right] \leq C'_\kappa \left(\frac{1}{(N \lambda)^2} + \frac{{\mathcal N}(\lambda)}{N\lambda}\right) \biggl\{\frac{\|f_\lambda-f_\rho\|_\rho}{\sqrt{N\lambda}} \sum_{j=1}^m \left(\frac{|D|}{|D_j|}\right)^{\frac{3}{2}} + M \sqrt{\lambda} \sqrt{\frac{{\mathcal N}(\lambda)}{N \lambda}}\biggr\} \\
&&+  C'_\kappa M \sqrt{\lambda} \left\{\sum_{j=1}^m \left(\frac{|D_j|}{|D|}\right)^2
\left(\frac{1}{|D_j|^2 \lambda^2} + \frac{{\mathcal N}(\lambda)}{|D_j|\lambda}\right)^2 \left\{1+ \left(\frac{1}{|D_j|^2 \lambda^2} + \frac{{\mathcal N}(\lambda)}{|D_j|\lambda}\right)^2 \right\}\right\}^{1/2}
\end{eqnarray*}
and
\begin{eqnarray*}
&& E\left[\left\|\overline{f}_{D, \lambda} - f_{D, \lambda}\right\|_K\right] \leq C'_\kappa \left(\frac{1}{(N \lambda)^2} + \frac{{\mathcal N}(\lambda)}{N\lambda}\right) \biggl\{\frac{\|f_\lambda-f_\rho\|_\rho}{\sqrt{N}\lambda} \sum_{j=1}^m \left(\frac{|D|}{|D_j|}\right)^{\frac{3}{2}} + M \sqrt{\frac{{\mathcal N}(\lambda)}{N \lambda}}\biggr\} \\
&&+  C'_\kappa M  \left\{\sum_{j=1}^m \left(\frac{|D_j|}{|D|}\right)^2
\left(\frac{1}{|D_j|^2 \lambda^2} + \frac{{\mathcal N}(\lambda)}{|D_j|\lambda}\right)^2 \left\{1+ \left(\frac{1}{|D_j|^2 \lambda^2} + \frac{{\mathcal N}(\lambda)}{|D_j|\lambda}\right)^2 \right\}\right\}^{1/2},
\end{eqnarray*}
where $C'_{\kappa}$ is a constant depending only on $\kappa$.
\end{theorem}

\begin{proof}
Recall (\ref{expressdiffLK}) in Lemma \ref{difference}. It enables us to express
\begin{equation}\label{fDrhodecom}
L_K^{1/2}\left\{\overline{f}_{D, \lambda} - f_{D, \lambda}\right\} = J_1 + J_2 + J_3,
\end{equation}
where the terms $J_1, J_2, J_3$ are given by
$$J_1 = \sum_{j=1}^m \frac{|D_j|}{|D|} \left[L_K^{1/2}Q_{D_j(x)}\right] \Delta'_j, \ J_2 = \sum_{j=1}^m \frac{|D_j|}{|D|} \left[L_K^{1/2}Q_{D_j(x)}\right] \Delta''_j, \ J_3 = -\left[L_K^{1/2}Q_{D(x)}\right] \Delta_D. $$
These three terms will be dealt with separately in the following.

For the first term $J_1$ of (\ref{fDrhodecom}), each summand with $j\in \{1, \ldots, m\}$ can be expressed as $\sum_{z\in D_j} \frac{1}{|D|} (y - f_\rho (x))\left[L_K^{1/2}Q_{D_j(x)}\right] (K_x)$, and is unbiased because $\int_{{\mathcal Y}} y- f_\rho (x)  d \rho(y|x) =0$. The unbiasedness and the independence tell us that
\begin{eqnarray*}
&& E\left[\left\|J_1\right\|_K\right] \leq \left\{E\left[\left\|J_1\right\|_K^2\right]\right\}^{1/2} \leq \left\{\sum_{j=1}^m \left(\frac{|D_j|}{|D|}\right)^2
E\left[\left\|\left[L_K^{1/2}Q_{D_j(x)}\right] \Delta'_j\right\|_K^2\right]\right\}^{1/2}.
\end{eqnarray*}
Let $j\in \{1, \ldots, m\}$. The relation (\ref{fDlambdafirst}) derived from the second order decomposition (\ref{secondorder}) in the proof of Proposition \ref{MainPropls} yields
\begin{equation}\label{primeone}
\left\|\left[L_K^{1/2}Q_{D_j(x)}\right] \Delta'_j\right\|_K^2 \leq \left(\frac{\Xi_{D_j}}{\sqrt{\lambda}} + \frac{\Xi_{D_j}^2}{\lambda}\right)^2 \left\|\left(L_{K} +\lambda I\right)^{-1/2} \Delta'_j\right\|_K^2.
\end{equation}
Now we apply the formula
\begin{equation}\label{Expectedform}
E[\xi] =\int_0^\infty \hbox{Prob}\left[\xi > t\right] d t
\end{equation}
to estimate the expected value of (\ref{primeone}). By Part (b) of
Lemma \ref{keynormsoperator}, for $0< \delta <1$, there exists a
subset ${\mathcal Z}^{|D_j|}_{\delta, 1}$ of ${\mathcal Z}^{|D_j|}$
of measure at least $1-\delta$ such that
\begin{equation}\label{XiDjest}
\Xi_{D_j} \leq {\mathcal B}_{|D_j|, \lambda} \log \bigl(2/\delta\bigr), \qquad \forall D_j \in {\mathcal Z}^{|D_j|}_{\delta, 1}.
\end{equation}
Applying Part (b) of Lemma \ref{keynorms} to $g(z) = y- f_\rho (x)$ with $\|g\|_\infty \leq 2 M$ and the data subset $D_j$, we know that there exists another subset ${\mathcal Z}^{|D_j|}_{\delta, 2}$ of ${\mathcal Z}^{|D_j|}$ of measure at least $1-\delta$ such that
$$
\left\|\left(L_{K} +\lambda I\right)^{-1/2} \Delta'_j\right\|_K \leq \frac{2 M}{\kappa} {\mathcal B}_{|D_j|, \lambda} \log \bigl(2/\delta\bigr), \qquad \forall D_j \in {\mathcal Z}^{|D_j|}_{\delta, 2}.
$$
Combining this with (\ref{XiDjest}) and (\ref{primeone}), we know that for $D_j \in {\mathcal Z}^{|D_j|}_{\delta, 1} \cap {\mathcal Z}^{|D_j|}_{\delta, 2}$,
$$
\left\|\left[L_K^{1/2}Q_{D_j(x)}\right] \Delta'_j\right\|_K^2 \leq  \left(\frac{{\mathcal B}_{|D_j|, \lambda}^2}{\lambda} + \frac{{\mathcal B}_{|D_j|, \lambda}^4}{\lambda^2}\right) \left(\frac{M}{\kappa}\right)^2 {\mathcal B}_{|D_j|, \lambda}^2 \left(2 \log \bigl(2/\delta\bigr)\right)^6.
$$
Since the measure of the set ${\mathcal Z}^{|D_j|}_{\delta, 1} \cap {\mathcal Z}^{|D_j|}_{\delta, 2}$ is at least $1-2\delta$, by denoting
$${\mathcal C}_{|D_j|, \lambda} =64 \left(\frac{{\mathcal B}_{|D_j|, \lambda}^2}{\lambda} + \frac{{\mathcal B}_{|D_j|, \lambda}^4}{\lambda^2}\right) \left(\frac{M}{\kappa}\right)^2 {\mathcal B}_{|D_j|, \lambda}^2, $$
we see that
$$  \hbox{Prob}\left[\left\|\left[L_K^{1/2}Q_{D_j(x)}\right] \Delta'_j\right\|_K^2 > {\mathcal C}_{|D_j|, \lambda} \left(\log \bigl(2/\delta\bigr)\right)^6\right] \leq 2 \delta. $$
For $0< t <\infty$, the equation ${\mathcal C}_{|D_j|, \lambda} \left(\log \bigl(2/\delta\bigr)\right)^6 =t$ has the solution
$$\delta_t = 2 \exp\left\{-\left(t/{\mathcal C}_{|D_j|, \lambda}\right)^{1/6}\right\}. $$
When $\delta_t <1$, we have
$$ \hbox{Prob}\left[\left\|\left[L_K^{1/2}Q_{D_j(x)}\right] \Delta'_j\right\|_K^2 > t\right] \leq 2 \delta_t = 4 \exp\left\{-\left(t/{\mathcal C}_{|D_j|, \lambda}\right)^{1/6}\right\}. $$
This inequality holds trivially when $\delta_t \geq 1$ since the probability is at most $1$. Thus we can apply the formula (\ref{Expectedform}) to  the nonnegative random variable $\xi = \left\|\left[L_K^{1/2}Q_{D_j(x)}\right] \Delta'_j\right\|_K^2$ and obtain
$$ E\left[\left\|\left[L_K^{1/2}Q_{D_j(x)}\right] \Delta'_j\right\|_K^2\right] = \int_0^\infty \hbox{Prob}\left[\xi > t\right] d t \leq \int_0^\infty 4 \exp\left\{-\left(t/{\mathcal C}_{|D_j|, \lambda}\right)^{1/6}\right\} d t $$
which equals $24 \Gamma (6) {\mathcal C}_{|D_j|, \lambda}$. Therefore,
\begin{eqnarray*}
&& E\left[\left\|J_1\right\|_K\right] \leq \left\{\sum_{j=1}^m \left(\frac{|D_j|}{|D|}\right)^2
24 \Gamma (6) {\mathcal C}_{|D_j|, \lambda} \right\}^{1/2} \\
&& \leq 1536\sqrt{5} \kappa M \left\{\sum_{j=1}^m \left(\frac{|D_j|}{|D|}\right)^2 \lambda
\left(\frac{\kappa^2}{|D_j|^2 \lambda^2} + \frac{{\mathcal N}(\lambda)}{|D_j|\lambda}\right)^2 \left\{1+ 8 \kappa^2 \left(\frac{\kappa^2}{|D_j|^2 \lambda^2} + \frac{{\mathcal N}(\lambda)}{|D_j|\lambda}\right)^2 \right\}\right\}^{1/2}.
\end{eqnarray*}

For the second term $J_2$ of (\ref{fDrhodecom}), we use the second order decomposition (\ref{secondorder}) again and obtain
$$
\left\|\left[L_K^{1/2}Q_{D_j(x)}\right] \Delta''_j\right\|_K \leq \left(\frac{\Xi_{D_j}}{\sqrt{\lambda}} + \frac{\Xi_{D_j}^2}{\lambda}\right) \left\|\left(L_{K} +\lambda I\right)^{-1/2} \Delta''_j\right\|_K.
$$
Applying the Schwarz inequality and Lemmas \ref{keynormsoperator} and \ref{keynorms}, we get
\begin{eqnarray*}E\left[\left\|\left[L_K^{1/2}Q_{D_j(x)}\right] \Delta''_j\right\|_K\right]  &\leq& \left\{E\left[\left(\frac{\Xi_{D_j}}{\sqrt{\lambda}} + \frac{\Xi_{D_j}^2}{\lambda}\right)^2\right]\right\}^{1/2} \\
 && \frac{1}{\sqrt{|D_j|}}
\left\{E\left[\left(f_\rho (x) - f_\lambda (x)\right)^2
\left\|\left(L_{K} +\lambda I\right)^{-1/2} (K_x) \right\|_K^2\right]\right\}^{1/2} \\
&\leq& \left(\left\{\frac{\kappa^2 {\mathcal N}(\lambda)}{|D_j|\lambda}\right\}^{1/2} +
 \left\{\frac{49 {\mathcal B}_{|D_j|, \lambda}^4}{\lambda^2}\right\}^{1/2}\right) \frac{\kappa \|f_\rho - f_\lambda\|_\rho}{\sqrt{|D_j|\lambda}} \\
&\leq&  \left(\frac{59\kappa^4}{(|D_j|\lambda)^2} +\frac{59\kappa^2
{\mathcal N}(\lambda)}{|D_j|\lambda}\right)\frac{\kappa \|f_\rho -
f_\lambda\|_\rho}{\sqrt{|D_j|\lambda}}.
\end{eqnarray*}
It follows that
$$ E\left[\left\|J_2\right\|_K\right]
\leq  \left(\frac{59\kappa^4}{(|D| \lambda)^2} +\frac{59\kappa^2
{\mathcal N}(\lambda)}{|D|\lambda}\right) \sum_{j=1}^m
2\frac{|D|}{|D_j|} \frac{\kappa \|f_\rho -
f_\lambda\|_\rho}{\sqrt{|D_j|\lambda}}. $$

The last term $J_3$ of (\ref{fDrhodecom}) has been handled in the proof of Proposition \ref{MainPropls} by ignoring the summand $\left(L_{K}+\lambda
      I\right)^{-1} \Delta_D$ in the expression for $f_{D, \lambda} - f_{\lambda}$, and we find from the trivial bound $\left\|\sigma^2_\rho\right\|_\infty \leq 4M^2$ with $p=\infty$ that
$$E\left[\left\|J_3\right\|_K\right]
\leq  \left(\frac{59\kappa^4}{(|D|\lambda)^2} +\frac{59\kappa^2
{\mathcal N}(\lambda)}{|D|\lambda}\right) \left(2M
\left(\frac{{\mathcal N}(\lambda)}{|D|}\right)^{\frac{1}{2}}  +
\frac{\kappa \|f_\rho - f_\lambda\|_\rho}{\sqrt{|D|\lambda}}\right).
$$ Combining the above estimates for the three terms of
(\ref{fDrhodecom}), we see that the desired error bound in the
$L^2_{\rho_X}$ metric holds true.

The estimate in the ${\mathcal H}_K$ metric follows from the steps in deriving the error bound in the $L^2_{\rho_X}$ metric except that in the representation (\ref{fDrhodecom}) the operator
$L_K^{1/2}$ in the front disappears. This change gives an additional factor $1/\sqrt{\lambda}$, the bound for the operator $\left(L_{K} +\lambda I\right)^{-1/2}$, and proves the desired error bound in the ${\mathcal H}_K$ metric.
\end{proof}

\medskip

{\par \noindent{\bf Proof of Theorem \ref{Mainexpected}\ }}
Since $|D_j| = \frac{N}{m}$ for $j=1, \ldots, m$, the bound in Theorem \ref{Mainexpectedgeneral} in the $L^2_{\rho_X}$ metric can be simplified as
\begin{eqnarray*}
E\left[\left\|\overline{f}_{D, \lambda} - f_{D, \lambda}\right\|_\rho\right] &\leq& C'_\kappa \left(\frac{1}{(N \lambda)^2} + \frac{{\mathcal N}(\lambda)}{N\lambda}\right) \biggl\{\frac{\|f_\lambda-f_\rho\|_\rho}{\sqrt{N\lambda}} m^{\frac{5}{2}} + M \sqrt{\lambda} \sqrt{\frac{{\mathcal N}(\lambda)}{N \lambda}}\biggr\} \\
&&+  C'_\kappa M \frac{\sqrt{\lambda}}{\sqrt{m}}
\left(\frac{m^2}{N^2 \lambda^2} + \frac{m{\mathcal N}(\lambda)}{N\lambda}\right) \left\{1+ \left(\frac{m^2}{N^2 \lambda^2} + \frac{m{\mathcal N}(\lambda)}{N\lambda}\right)\right\} \\
&\leq& C'_\kappa \left(\frac{m}{(N \lambda)^2} + \frac{{\mathcal N}(\lambda)}{N\lambda}\right) \biggl\{\frac{\|f_\lambda-f_\rho\|_\rho}{\sqrt{N\lambda}} m^{\frac{5}{2}} + M \sqrt{\lambda} \sqrt{\frac{{\mathcal N}(\lambda)}{N \lambda}}  \\
&&+  M \sqrt{\lambda} \sqrt{m} \left\{1+ \left(\frac{m^2}{N^2 \lambda^2} + \frac{m{\mathcal N}(\lambda)}{N\lambda}\right)\right\}\biggr\}.
\end{eqnarray*}
Notice that the term
$\sqrt{\frac{{\mathcal N}(\lambda)}{N \lambda}}$ can be bounded by $1+ \frac{m{\mathcal N}(\lambda)}{N\lambda}$.
Then the desired error bound in the $L^2_{\rho_X}$ metric with $C_\kappa =  2 C'_\kappa$ follows. The proof for the error bound in the ${\mathcal H}_K$ metric is similar.
The proof of Theorem \ref{Mainexpected} is complete. \hfill
\BlackBox

\medskip

{\par \noindent{\bf Proof of Corollary \ref{MainEqualsize}\ }}
As in the proof of Corollary \ref{Specialpara}, the restriction (\ref{lambdacond}) implies ${\mathcal N}(\lambda) \geq \frac{\lambda_1}{\lambda_1 + C_0}$ and
$N\lambda \geq \frac{m \lambda_1}{(\lambda_1 + C_0) C_0}$. It follows that
$$ \frac{m}{(N \lambda)^2} \leq \frac{(\lambda_1 + C_0) C_0}{\lambda_1} \frac{1}{N\lambda} \leq \frac{(\lambda_1 + C_0)^2 C_0^2}{\lambda_1^2} \frac{{\mathcal N}(\lambda)}{N\lambda}. $$
Putting these bounds into Theorem \ref{Mainexpected}, we know that the expected value $E\left\|\overline{f}_{D, \lambda} - f_{D, \lambda}\right\|_\rho$ is bounded by
\begin{eqnarray*}
 && C_\kappa \left(\frac{(\lambda_1 + C_0)^2 C_0^2}{\lambda_1^2} + 1\right) \frac{{\mathcal N}(\lambda)}{N\lambda} \sqrt{m} \biggl\{\frac{\|f_\lambda-f_\rho\|_\rho}{\sqrt{N\lambda}} m^2 +  M \sqrt{\lambda} \left\{1+ \frac{(\lambda_1 + C_0)^2 C_0^2}{\lambda_1^2} + C_0\right\}\biggr\} \\
&\leq& \widetilde{C}_\kappa \frac{m{\mathcal N}(\lambda)}{N\lambda} \biggl(\|f_\lambda-f_\rho\|_\rho \frac{m \sqrt{m}}{\sqrt{N\lambda}} +  M \frac{\sqrt{\lambda}}{\sqrt{m}}\biggr),
\end{eqnarray*}
and
$$ E\left\|\overline{f}_{D, \lambda} - f_{D, \lambda}\right\|_K
\leq \widetilde{C}_\kappa \frac{m{\mathcal N}(\lambda)}{N\lambda} \biggl(\|f_\lambda-f_\rho\|_\rho \frac{m \sqrt{m}}{\sqrt{N}\lambda} +  \frac{M}{\sqrt{m}}\biggr),$$
where
$$ \widetilde{C}_\kappa := C_\kappa \left(\frac{(\lambda_1 + C_0)^2 C_0^2}{\lambda_1^2} + 1\right)\left\{1+ \frac{(\lambda_1 + C_0)^2 C_0^2}{\lambda_1^2} + C_0\right\}. $$
This proves Corollary \ref{MainEqualsize}. \hfill
\BlackBox

\medskip

{\par \noindent{\bf Proof of Corollary \ref{SpecialEqualsize}\ }}
If
$${\mathcal N}(\lambda) \leq C_0 \lambda^{-\frac{1}{2\alpha}}, \qquad \forall \lambda >0 $$
for some constant $C_0 \geq 1$, then the choice $\lambda = \left(\frac{m}{N}\right)^{\frac{2\alpha}{2 \alpha \max\{2r, 1\} +1}}$ satisfies (\ref{lambdacond}). With this choice we also have
\begin{eqnarray*} \frac{m{\mathcal N}(\lambda)}{N\lambda} \leq
C_0 \left(\frac{m}{N}\right)^{\frac{2\alpha (\max\{2r, 1\}-1)}{2 \alpha \max\{2r, 1\} +1}}.
\end{eqnarray*}
Since the condition (\ref{approxr}) yields
$\|f_\lambda - f_\rho\|_\rho \leq \|g_\rho\|_\rho \lambda^{r}$, we have by Corollary \ref{MainEqualsize},
\begin{eqnarray*}
E\left\|\overline{f}_{D, \lambda} - f_{D, \lambda}\right\|_\rho &\leq& \widetilde{C}_\kappa C_0 \left(\frac{m}{N}\right)^{\frac{2\alpha (\max\{2r, 1\}-1)}{2 \alpha \max\{2r, 1\} +1}} \biggl(\|g_\rho\|_\rho \left(\frac{m}{N}\right)^{\frac{2\alpha}{2 \alpha \max\{2r, 1\} +1}(r-\frac{1}{2})} \frac{m \sqrt{m}}{\sqrt{N}} \\
&& +  M N^{-\frac{\alpha}{2 \alpha \max\{2r, 1\}+1}} m^{-\frac{2\alpha (\max\{2r, 1\}-1) +1}{2 + 4 \alpha\max\{2r, 1\}}}\biggr).
\end{eqnarray*}
The inequality $\left(\frac{m}{N}\right)^{\frac{2\alpha}{2 \alpha \max\{2r, 1\} +1}(r-\frac{1}{2})} \frac{m \sqrt{m}}{\sqrt{N}}  \leq N^{-\frac{\alpha}{2 \alpha \max\{2r, 1\}+1}} m^{-\frac{2\alpha (\max\{2r, 1\}-1) +1}{2 + 4 \alpha\max\{2r, 1\}}}$ is equivalent to
$$m^{\frac{3}{2} +\frac{2\alpha (\max\{2r, 1\}-1) +1}{2 + 4 \alpha\max\{2r, 1\}} + \frac{2\alpha}{2 \alpha \max\{2r, 1\} +1}(r-\frac{1}{2})} \leq N^{\frac{1}{2}-\frac{\alpha}{2 \alpha \max\{2r, 1\}+1} + \frac{2\alpha}{2 \alpha \max\{2r, 1\} +1}(r-\frac{1}{2})}$$ and it can be expressed as (\ref{mrestrict}). Since (\ref{mrestrict}) is valid, we have
$$ E\left\|\overline{f}_{D, \lambda} - f_{D, \lambda}\right\|_\rho \leq \widetilde{C}_\kappa C_0 \left(\frac{m}{N}\right)^{\frac{2\alpha (\max\{2r, 1\}-1)}{2 \alpha \max\{2r, 1\} +1}} \biggl(\|g_\rho\|_\rho +  M\biggr) N^{-\frac{\alpha}{2 \alpha \max\{2r, 1\}+1}} m^{-\frac{2\alpha (\max\{2r, 1\}-1) +1}{2 + 4 \alpha\max\{2r, 1\}}}. $$
This proves the first desired convergence rate. The second rate follows easily. This proves Corollary \ref{SpecialEqualsize}. \hfill
\BlackBox

\medskip

{\par \noindent{\bf Proof of Corollary \ref{Finalrate}\ }} By Corollary \ref{MainMinmax}, with the choice $\lambda = N^{-\frac{2\alpha}{4\alpha r +1}}$, we can immediately bound $\left\|f_{D, \lambda} - f_{\rho}\right\|_\rho$ as
$$
E\left[\left\|f_{D, \lambda} - f_{\rho}\right\|_\rho\right] = O\left(N^{-\frac{2 \alpha r}{4 \alpha r +1}}\right).
$$

The assumption ${\mathcal N}(\lambda) = O(\lambda^{-\frac{1}{2\alpha}})$ tells us that for some constant $C_0 \geq 1$,
$${\mathcal N}(\lambda) \leq C_0 \lambda^{-\frac{1}{2\alpha}}, \qquad \forall \lambda >0. $$
So the choice $\lambda = N^{-\frac{2\alpha}{4 \alpha r +1}}$ yields
\begin{equation}\label{restrictmcor}
 \frac{m {\mathcal N}(\lambda)}{N\lambda} \leq C_0 \frac{m \lambda^{-\frac{1 + 2 \alpha}{2\alpha}}}{N} =  C_0 m N^{\frac{1 + 2 \alpha}{4 \alpha r +1}-1} =  C_0 m N^{\frac{2 \alpha(1-2r)}{4 \alpha r +1}}.
\end{equation}
If $m$ satisfies
\begin{equation}\label{restrictmcor}
m \leq N^{\frac{2 \alpha(2r-1)}{4 \alpha r +1}},
\end{equation}
then (\ref{lambdacond}) is valid, and by Corollary \ref{MainEqualsize},
\begin{eqnarray*}
E\left\|\overline{f}_{D, \lambda} - f_{D, \lambda}\right\|_\rho &\leq& \widetilde{C}_\kappa \frac{m{\mathcal N}(\lambda)}{N\lambda} \biggl(\|f_\lambda-f_\rho\|_\rho  \frac{m \sqrt{m}}{\sqrt{N\lambda}} +  M \frac{\sqrt{\lambda}}{\sqrt{m}}\biggr) \\
&\leq& \widetilde{C}_\kappa  C_0 m N^{\frac{2 \alpha(1-2r)}{4 \alpha r +1}} \biggl(\lambda^r \|g_\lambda\|_\rho  \frac{m \sqrt{m}}{\sqrt{N\lambda}} +  M \frac{\sqrt{\lambda}}{\sqrt{m}}\biggr) \\
&\leq& \widetilde{C}_\kappa  C_0 \biggl(\|g_\lambda\|_\rho  +  M\biggr) \lambda^r \left(\frac{m^{\frac{5}{2}} N^{\frac{2 \alpha(1-2r)}{4 \alpha r +1}}}{\sqrt{N\lambda}} + \sqrt{m} N^{\frac{2 \alpha(1-2r)}{4 \alpha r +1}} \lambda^{\frac{1}{2} -r}\right) \\
&=& \widetilde{C}_\kappa  C_0 \biggl(\|g_\lambda\|_\rho  +  M\biggr) N^{-\frac{2 \alpha r}{4\alpha r +1}} \left(m^{\frac{5}{2}} N^{-\frac{3 \alpha(2r-1) +\frac{1}{2}}{4 \alpha r +1}} + \sqrt{m} N^{-\frac{\alpha(2r -1)}{4 \alpha r +1}}\right).
\end{eqnarray*}
Thus, when $m$ satisfies
\begin{equation}\label{restrictmIIcor}
m \leq N^{\frac{6 \alpha(2r-1) +1}{5(4 \alpha r +1)}}, \qquad m \leq N^{\frac{2\alpha(2r -1)}{4 \alpha r +1}},
\end{equation}
we have
$$ E\left\|\overline{f}_{D, \lambda} - f_{D, \lambda}\right\|_\rho \leq 2 \widetilde{C}_\kappa  C_0 \biggl(\|g_\lambda\|_\rho  +  M\biggr) N^{-\frac{2 \alpha r}{4\alpha r +1}}, $$
and thereby
$$ E\left[\left\|\overline{f}_{D, \lambda} - f_{\rho}\right\|_\rho\right] = O\left(N^{-\frac{2 \alpha r}{4 \alpha r +1}}\right).
$$
Finally, we notice that (\ref{restrictmfinal}) is equivalent to the combination of (\ref{restrictmcor}) and (\ref{restrictmIIcor}). So our conclusion follows.  This proves Corollary \ref{Finalrate}. \hfill
\BlackBox

\section*{Appendix}
To estimate  norms of various operators involving the approximation
of $L_K$ by $L_{K, D(x)}$, we need the following probability
inequality for vector-valued random variables in
\citep{Pinelis1994}.

\begin{lemma}\label{Pinelislemma}
For a random variable $\xi$ on $({\mathcal Z}, \rho)$ with values in
a Hilbert space $(H, \|\cdot\|)$ satisfying $\|\xi\| \leq
\widetilde{M} <\infty$ almost surely, and a random sample
$\{z_i\}_{i=1}^s$ independent drawn according to $\rho$, there holds
with confidence $1-\widetilde{\delta}$,
\begin{equation}\label{Pinelisthm}
\biggl\|{1 \over s} \sum_{i=1}^s \bigl[\xi (z_i) -
E(\xi)\bigr]\biggr\| \leq {2 \widetilde{M} \log
\bigl(2/\widetilde{\delta}\bigr) \over s}  + \sqrt{{2 E (\|\xi\|^2)
\log \bigl(2/\widetilde{\delta}\bigr) \over s}}.
\end{equation}
\end{lemma}

{\par \noindent{\bf Proof of Lemma \ref{keynormsoperator}\ }}  We
apply Lemma (\ref{Pinelislemma}) to the random variable $\eta_1$
defined by
\begin{equation}\label{eta1}
\eta_1 (x) =\left(L_K+\lambda I\right)^{-1/2} \langle \cdot,
K_x\rangle_K K_{x}, \qquad x\in {\mathcal X}
\end{equation}
It takes values in $HS({\mathcal H}_K)$, the Hilbert space of
Hilbert-Schmidt operators on ${\mathcal H}_K$, with inner product
$\langle A, B\rangle_{HS} = \hbox{Tr}(B^T A).$ Here Tr denotes the
trace of a (trace-class) linear operator. The norm is given by
$\|A\|_{HS}^2 =\sum_{i} \|A e_i\|_K^2$ where $\{e_i\}$ is an
orthonormal basis of ${\mathcal H}_K$. The space $HS({\cal H}_K)$ is
a subspace of the space of bounded linear operators on ${\cal H}_K$,
denoted as $(L({\cal H}_K), \|\cdot\|)$, with the norm relations
\begin{equation}\label{normrelation}
 \|A\| \leq \|A\|_{HS},
\qquad \|A B \|_{HS} \leq \|A\|_{HS} \|B\|.
\end{equation}

Now we use effective dimensions to estimate norms involving
$\eta_1$. The random variable $\eta_1$ defined by (\ref{eta1}) has
mean $E(\eta_1) = \left(L_K+\lambda I\right)^{-1/2} L_K$ and sample
mean $\left(L_K+\lambda I\right)^{-1/2} L_{K, D (x)}$. Recall the
set of normalized (in ${\mathcal H}_K$) eigenfunctions
$\{\varphi_i\}_{i}$ of $L_K$. It is an orthonormal basis of
${\mathcal H}_K$. If we regard $L_K$ as an operator on
$L^2_{\rho_X}$, the normalized eigenfunctions in $L^2_{\rho_X}$ are
$\{\frac{1}{\sqrt{\lambda_i}} \varphi_i\}_{i}$ and they form an
orthonormal basis of the orthogonal complement of the eigenspace
associated with eigenvalue $0$. By the Mercer Theorem, we have the
following uniform convergent Mercer expansion
\begin{equation}\label{Mercer}
 K(x, y) = \sum_{i} \lambda_i \frac{1}{\sqrt{\lambda_i}} \varphi_i (x) \frac{1}{\sqrt{\lambda_i}}\varphi_i (y) =
\sum_{i} \varphi_i (x) \varphi_i (y).
\end{equation}
Take the orthonormal basis $\{\varphi_i\}_{i}$ of ${\mathcal H}_K$.
By the definition of the HS norm, we have
$$ \|\eta_1 (x)\|_{HS}^2 = \sum_i \left\|\left(L_K+\lambda I\right)^{-1/2}
\langle \cdot, K_x\rangle_K K_{x} \varphi_i\right\|_K^2. $$ For a
fixed $i$,
$$ \langle \cdot, K_x\rangle_K K_{x} \varphi_i = \varphi_i (x) K_x, $$
and $K_x \in {\mathcal H}_K$ can be expended by the orthonormal
basis $\{\varphi_\ell\}_{\ell}$ as
\begin{equation}\label{expandKx}
K_x = \sum_\ell \langle \varphi_\ell, K_x\rangle_K \varphi_\ell =
\sum_\ell \varphi_\ell (x) \varphi_\ell.
\end{equation}
Hence
\begin{eqnarray*} \|\eta_1 (x)\|_{HS}^2 &=& \sum_i \left\|\varphi_i (x) \sum_\ell \varphi_\ell (x) \left(L_K+\lambda I\right)^{-1/2} \varphi_\ell\right\|_K^2 \\
&=& \sum_i \left\|\varphi_i (x) \sum_\ell \varphi_\ell (x)
\frac{1}{\sqrt{\lambda_\ell + \lambda}} \varphi_\ell\right\|_K^2 =
\sum_i \left(\varphi_i (x)\right)^2 \sum_\ell
\frac{\left(\varphi_\ell (x)\right)^2}{\lambda_\ell + \lambda}.
\end{eqnarray*}
Combining this with (\ref{Mercer}), we see that
\begin{equation}\label{eta1bound}
 \|\eta_1(x)\|_{HS}^2 = K(x, x) \sum_\ell \frac{\left(\varphi_\ell(x)\right)^2}{\lambda_\ell + \lambda}, \qquad \forall x\in {\mathcal X}
\end{equation}
and
\begin{eqnarray*}
 E\left[\|\eta_1(x)\|^2_{HS}\right]  \leq  \kappa^2 E\left[\sum_\ell \frac{\left(\varphi_\ell (x)\right)^2}{\lambda_\ell + \lambda}\right] = \kappa^2 \sum_\ell \frac{\int_{\mathcal X} \left(\varphi_\ell(x)\right)^2 d \rho_X}{\lambda_\ell + \lambda}.
\end{eqnarray*}
But
\begin{equation}\label{eigenvi}
 \int_{\mathcal X} \left(\varphi_\ell (x)\right)^2 d \rho_X  = \left\|\varphi_\ell\right\|^2_{L^2_{\rho_X}} = \left\|\sqrt{\lambda_\ell} \frac{1}{\sqrt{\lambda_\ell}}\varphi_\ell\right\|^2_{L^2_{\rho_X}} = \lambda_\ell.
 \end{equation}
So we have
\begin{equation}\label{Eeta1} E\left[\|\eta_1\|^2_{HS}\right]  \leq \kappa^2 \sum_\ell \frac{\lambda_\ell}{\lambda_\ell + \lambda}
= \kappa^2 \hbox{Tr}\left(\left(L_K+\lambda I\right)^{-1} L_K\right)
= \kappa^2 {\mathcal N}(\lambda)
\end{equation}
and
$$ E\left\|\frac{1}{|D|} \sum_{x\in D (x)} \eta_1 (x) - E[\eta_1]\right\|^2_{HS} = E\left[\left\|\left(L_{K} +\lambda I\right)^{-1/2} \left\{L_K- L_{K, D(x)}\right\}\right\|_{HS}^2\right] \leq \frac{\kappa^2 {\mathcal N}(\lambda)}{|D|}. $$
Then our desired inequality in Part (a) follows from the first
inequality of (\ref{normrelation}).

From (\ref{expandKx}) and (\ref{eta1bound}), we find a bound for
$\eta_1$ as
$$ \|\eta_1 (x)\|_{HS} \leq \kappa \frac{1}{\sqrt{\lambda}} \sqrt{\sum_\ell \left(\varphi_\ell(x)\right)^2} \leq  \frac{\kappa}{\sqrt{\lambda}} \sqrt{K (x, x)} \leq  \frac{\kappa^2}{\sqrt{\lambda}}, \qquad \forall x\in {\mathcal X}. $$
Applying Lemma \ref{Pinelislemma} to the random variable $\eta_1$
with $\widetilde{M} = \frac{\kappa^2}{\sqrt{\lambda}}$, we know by
(\ref{normrelation}) that with confidence at least $1-\delta$,
\begin{eqnarray*}
\left\|E[\eta_1] - \frac{1}{|D|} \sum_{x\in D(x)} \eta_1 (x)\right\| &\leq& \left\|E[\eta_1] - \frac{1}{|D|} \sum_{x\in D(x)} \eta_1 (x)\right\|_{HS} \\
 &\leq& {2 \kappa^2 \log \bigl(2/\delta\bigr) \over |D|\sqrt{\lambda}}  + \sqrt{\frac{2 \kappa^2 {\mathcal N}(\lambda)\log \bigl(2/\delta\bigr)}{|D|}}.
\end{eqnarray*}
Writing the above bound by taking a factor $\frac{2 \kappa \log
\bigl(2/\delta\bigr)}{\sqrt{|D|}}$, we get the desired bound
(\ref{normprepareeta2}).

Recall ${\mathcal B}_{|D|, \lambda}$ defined by (\ref{Bdef}). Apply
the formula (\ref{Expectedform}) for nonnegative random variables to
$\xi = \left\|\left(L_{K} +\lambda I\right)^{-1/2} \left\{L_K- L_{K,
D(x)}\right\}\right\|^d$ and use the bound
$$ \hbox{Prob}\left[\xi > t\right] = \hbox{Prob}\left[\xi^{\frac{1}{d}} > t^{\frac{1}{d}}\right] \leq 2 \exp\left\{-\frac{t^{\frac{1}{d}}}{{\mathcal B}_{|D|, \lambda}}\right\} $$
derived from (\ref{normprepareeta2}) for $t\geq\log^d 2\mathcal
B_{|D|,\lambda}$. We find
$$ E\left[\left\|\left(L_{K} +\lambda I\right)^{-1/2}
 \left\{L_K- L_{K, D(x)}\right\}\right\|^d\right] \leq \log^d 2\mathcal
B_{|D|,\lambda}+ \int_0^\infty 2
\exp\left\{-\frac{t^{\frac{1}{d}}}{{\mathcal B}_{|D|,
\lambda}}\right\} d t. $$ The second term in the right hand of above
equation equals $2 d {\mathcal B}_{|D|, \lambda}^d \int_0^\infty
u^{d-1} \exp\left\{-u\right\} d u$. Then the desired bound in Part
(c) follows from $\int_0^\infty u^{d-1} \exp\left\{-u\right\} d u =
\Gamma (d)$ and the lemma is proved.\hfill \BlackBox

{\par \noindent{\bf Proof of Lemma \ref{keynorms}\ }} Consider the
random variable $\eta_2$ defined by
\begin{equation}\label{eta2}
\eta_2 (z) =\left(L_K+\lambda I\right)^{-1/2}\left(K_{x}\right),
\qquad z=(x, y)\in {\mathcal Z}.
\end{equation}
It takes values in ${\mathcal H}_K$. By (\ref{expandKx}), it
satisfies
$$ \left\|\eta_2 (z)\right\|_K = \left\|\left(L_K+\lambda I\right)^{-1/2} \left(\sum_\ell \varphi_\ell (x) \varphi_\ell\right)\right\|_K = \left(\sum_\ell \frac{\left(\varphi_\ell (x)\right)^2}{\lambda_\ell + \lambda}\right)^{1/2}. $$
So
$$ E\left[\left\|\left(L_{K} +\lambda I\right)^{-1/2} (K_x)\right\|_K^2\right] = E\left[\sum_{\ell} \frac{(\varphi_\ell (x))^2}{\lambda_\ell + \lambda}\right] = {\mathcal N}(\lambda). $$
This is the statement of Part (a).

For Part (b), we consider another random variable $\eta_3$ defined
by
\begin{equation}\label{eta2}
\eta_3 (z) =\left(L_K+\lambda I\right)^{-1/2}\left(g(z)
K_{x}\right), \qquad z=(x, y)\in {\mathcal Z}.
\end{equation}
It takes values in ${\mathcal H}_K$ and satisfies
$$ \left\|\eta_3 (z)\right\|_K = |g(z)| \left\|\left(L_K+\lambda I\right)^{-1/2} \left(K_x\right)\right\|_K = |g(z)| \left(\sum_\ell \frac{\left(\varphi_\ell (x)\right)^2}{\lambda_\ell + \lambda}\right)^{1/2}. $$
So
$$ \left\|\eta_3 (z)\right\|_K \leq \frac{\kappa \|g\|_\infty}{\sqrt{\lambda}}, \qquad z\in {\mathcal Z}  $$
and
$$ E\left[\|\eta_3\|^2_{K}\right]  \leq \|g\|_\infty^2 E\left[\sum_\ell \frac{\left(\varphi_\ell (x)\right)^2}{\lambda_\ell + \lambda}\right] = \|g\|_\infty^2 {\mathcal N}(\lambda).
$$
Applying Lemma \ref{Pinelislemma} proves the statement in Part (b).
\hfill \BlackBox

\acks{Three anonymous referees and the action editor have carefully
read the paper and have provided to us numerous constructive
suggestions. As a result, the overall quality of the paper has been
noticeably enhanced, to which we feel much indebted and are
grateful. The work described in this paper  is supported partially
by the Research Grants Council of Hong Kong [Project No. CityU
11304114]. The corresponding author is Ding-Xuan Zhou.}


\begin{thebibliography}{999}

\bibitem[Bach(2013)]{Bach2013}
F. Bach. Sharp analysis of low-rank kernel matrix approximations.
ArXiv:1208.2015, 2013.

\bibitem[Bauer et al.(2007)]{BPR}
F. Bauer, S. Pereverzev, and L. Rosasco. On regularization
algorithms in learning theory.
 {\em Journal of Complexity}, 23:52-72, 2007.

 \bibitem[Blanchard and Kr$\ddot{\mbox{a}}$mer(2010)]{Blanchard2010}
G. Blanchard and N. Kr$\ddot{\mbox{a}}$mer.  Optimal learning rates
for kernel conjugate gradient regression. {\em Advances in Neural
Information Processing Systems},   226-234, 2010.

\bibitem[Caponnetto and De Vito(2007)]{Caponnetto2007}
 A. Caponnetto and E. DeVito. Optimal rates for the regularized least
squares algorithm. {\em Foundations of  Computational Mathematics},
7:331-368, 2007.

\bibitem[Chen et al.(2004)]{CWYZ} D. R. Chen, Q. Wu, Y. Ying, and D. X.
Zhou. Support vector machine soft margin classifiers: error
analysis. {\em Journal of  Machine Learning Research}, 5:1143--1175,
2004.

\bibitem[Cristianini and Shawe-Taylor(2000)]{Taylor2004}
N. Cristianini and J. Shawe-Taylor. {\em An Introduction to Support
Vector Machines}. Cambridge University Press, 2000.

\bibitem[Dekel et al.(2012)]{Dekel2012}
O. Dekel, R. Gilad-Bachrach, O. Shamir, and X. Lin. Optimal
distributed online prediction using mini-batches. {\em Journal of
Machine Learning Research},   13:165-202, 2012.

\bibitem[De Vito et al.(2005)]{Devito2005}
E. De Vito, A. Caponnetto, and L. Rosasco. Model selection for
regularized least-squares algorithm in learning theory. {\em
Foundations of  Computational Mathematics}, 5:59-85, 2005.

\bibitem[De Vito et al.(2010)]{Devito2010}
E. De Vito, S. Pereverzyev, and L. Rosasco. Adaptive kernel methods
using the balancing principle. {\em Foundations of Computational
Mathematics} 10:455-479, 2010.

\bibitem[Edmunds and Triebel(1996)]{EdmundsTriebel}
D.E. Edmunds and H. Triebel.
Function spaces, entropy numbers, differential operators.
{\em Cambridge University Press, Cambridge}, 1996.

\bibitem[Evgeniou et al.(2000)]{Evgeniou2000}
T. Evgeniou, M. Pontil,  and T. Poggio. Regularization networks and
support vector machines. {\em Advance in Computional  Mathematics},
13:1-50, 2000.

\bibitem[Fine(2002)]{Fine2002}
S. Fine  and K. Scheinberg. Efficient SVM training using low-rank
kernel representations. {\em Journal of  Machine Learning Research},
2:243-264, 2002.

\bibitem[Gittens and Mahoney(2016)]{Gittens2016}
A. Gittens and M. Mahoney. Revisiting the Nystr\"{o}m Method for
Improved Large-scale Machine Learning. {\em Journal of  Machine
Learning Research},  17:1-65, 2016.



\bibitem[Guo and Zhou(2012)]{GuoZhou1}
X. Guo and D. X. Zhou. An empirical feature-based learning algorithm
producing sparse approximations. {\em Applied and Computational Harmonic Analysis}, 32:389-400, 2012.


\bibitem[Gy\"{o}rfy et al.(2002)]{Gyorfi2002}
L. Gy\"{o}rfy, M. Kohler, A. Krzyzak,  H. Walk. A Distribution-Free
Theory of Nonparametric Regression. Springer-Verlag, Berlin, 2002.

\bibitem[Hu et al.(2015)]{HFWZ} T. Hu, J. Fan, Q. Wu, and D. X. Zhou.
Regularization schemes for minimum error entropy principle.
 {\em Analysis and Applications}, 13:437--455, 2015.

\bibitem[Lin and Zhou(2015)]{Linjh2015}
J. H. Lin and D. X. Zhou. Learning theory of randomized Kaczmarz
algorithm. {\em Journal of  Machine Learning Research},
16:3341-3365, 2015.

\bibitem[Mendelson and Neeman(2010)]{Mendelson2010}
S. Mendelson and J. Neeman. Regularization in kernel learning. {\em
The Annals of Statistics},  38(1):526-565, 2010.


\bibitem[Meister and Steinwart(2016)]{Meister2016}
M. Meister, I. Steinwart. Optimal Learning Rates for Localized SVMs.
{\em Journal of Machine Learning Research},  17: 1-44, 2016.



\bibitem[Pinelis(1994)]{Pinelis1994}
I. Pinelis.  Optimum bounds for the distributions of martingales in
Banach spaces. {\em The Annals of  Probability}, 22:1679-1706, 1994.


\bibitem[Raskutti et al.(2014)]{Raskutti2014}
G. Raskutti, M. Wainwright, and B. Yu. Early stopping and
non-parametric regression: an optimal data-dependent stopping rule.
{\em Journal of  Machine Learning Research}, 15:335-366, 2014.

\bibitem[Sch\"{o}lkopf et al.(1998)]{Scholkopf1998}
B. Sch\"{o}lkopf, A. Smola, and K. R. M\"{u}ller. Nonlinear
component analysis as a kernel eigenvalue problem. {\em IEEE
Transactions on Information Theory}, 10:1299-1319, 1998.

\bibitem[Shamir and Srebro(2014)]{Shamir2014}
O. Shamir  and  N. Srebro. Distributed stochastic optimization and
learning. In 52nd Annual Allerton Conference on Communication,
Control and Computing, 2014.

\bibitem[Shen et al.(2014)]{ShenWongXGSPeptide}
W.J. Shen, H.S. Wong, Q.W. Xiao, X. Guo, and S. Smale.
Introduction to the Peptide Binding Problem of Computational Immunology: New Results.
{\em Foundations of Computational Mathematics}, 14:951-984, 2014.


\bibitem[Smale and Zhou(2007)]{Smale2007}
S. Smale and D.X. Zhou. Learning theory estimates via integral operators and their
approximations. {\em Constructive Approximation}, 26:153-172, 2007.


\bibitem[Steinwart and Christmann(2008)]{Steinwart2008}
I. Steinwart, A. Christmann, {\em Support Vector Machines}. Springer, New
York, 2008.

\bibitem[Steinwart et al.(2009)]{SteinwartHS}
I. Steinwart, D. Hush, and C. Scovel. Optimal rates for regularized
least squares regression. in Proceedings of the 22nd Annual
Conference on Learning Theory (S. Dasgupta and A. Klivans, eds.),
pp. 79-93, 2009.

\bibitem[Steinwart and Scovel(2012)]{Steinwart2012}
I. Steinwart and C. Scovel. Mercer's theorem on general domains: On
the interaction between measures, kernels, and RKHSs. {\em
Constructive Approximation}, 35(3):363-417, 2012.

\bibitem[Wu et al.(2006)]{WYZ2006} Q. Wu, Y. Ying, and D. X. Zhou. Learning rates of
least-square regularized regression. {\em Foundations of
Computational Mathematics}, 6:171--192, 2006.

\bibitem[Yao et al.(2007)]{Yao2007}
Y. Yao, L. Rosasco, and A. Caponnetto. On early stopping in gradient
descent learning. {\em Constructive Approximation}, 26:289-315,
2007.

\bibitem[Zhang(2005)]{Zhang2005}
T. Zhang. Learning bounds for kernel regression using effective data
dimensionality. {\em Neural Computation}, 17:2077-2098, 2005.

\bibitem[Zhang et al.(2013)]{Zhang2013}
Y. C. Zhang, J. Duchi,  and M. Wainwright. Communication-efficient
algorithms for statistical optimization. {\em Journal of  Machine
Learning Research}, 14:3321-3363, 2013.

\bibitem[Zhang et al.(2015)]{Zhang2014}
Y. C. Zhang, J. Duchi,  and M. Wainwright. Divide and conquer kernel
ridge regression: A distributed algorithm with minimax optimal
rates. {\em Journal of Machine Learning Research}, 16:3299-3340,
2015.

\bibitem[Zhou(2002)]{Zhou02} D. X. Zhou. The covering number in learning
theory. {\em Journal of  Complexity}, 18:739--767, 2002.

\bibitem[Zhou(2003)]{Zhoucap} D. X. Zhou. Capacity of reproducing kernel spaces in learning theory.
 {\em IEEE Transactions on Information Theory}, 49:1743-1752, 2003.

\bibitem[Zhou et al.(2014)]{Zhou2014}
Z. H. Zhou, N. V. Chawla, Y. Jin,   G. J. Williams, Big data
opportunities and challenges: Discussions from data analytics
perspectives, {\em IEEE Computational Intelligence Magazine},  9:62-74, 2014.


\end{thebibliography}
\end{document}